\documentclass[11pt]{article}

\usepackage{amssymb}
\usepackage{amsmath}
\usepackage{amsthm}
\usepackage{accents}
\usepackage{enumerate}
\usepackage{algorithm}
\usepackage{graphicx}
\usepackage{url}
\usepackage{bm}

\usepackage{color}

\theoremstyle{plain}

\newtheorem{theo}{Theorem}
\newtheorem{prop}{Proposition}
\newtheorem{lemm}{Lemma}
\newtheorem{coro}{Corollary}

\newtheorem{assump}{Assumption}
\newtheorem{observation}{Observation}

\theoremstyle{definition}

\newtheorem{remark}{Remark}

\def\0{\bm{0}}
\def\1{\bm{1}}
\def\2{\bm{2}}
\def\3{\bm{3}}
\def\4{\bm{4}}
\def\5{\bm{5}}
\def\6{\bm{6}}
\def\7{\bm{7}}
\def\8{\bm{8}}
\def\9{\bm{9}}

\def\a{\bm{a}}
\def\b{\bm{b}}
\def\c{\bm{c}}

\def\e{\bm{e}}
\def\f{\bm{f}}

\def\k{\bm{k}}

\def\p{\bm{p}}
\def\q{\bm{q}}

\def\u{\bm{u}}
\def\v{\bm{v}}
\def\w{\bm{w}}
\def\x{\bm{x}}

\def\z{\bm{z}}

\def\A{\bm{A}}
\def\B{\bm{B}}

\def\D{\bm{D}}

\def\F{\bm{F}}
\def\G{\bm{G}}
\def\H{\bm{H}}
\def\I{\bm{I}}

\def\K{\bm{K}}
\def\L{\bm{L}}
\def\M{\bm{M}}
\def\N{\bm{N}}

\def\P{\bm{P}}
\def\Q{\bm{Q}}

\def\S{\bm{S}}

\def\U{\bm{U}}
\def\V{\bm{V}}
\def\W{\bm{W}}

\def\BC{\mathcal{B}}
\def\CC{\mathcal{C}}

\def\EC{\mathcal{E}}

\def\HC{\mathcal{H}}
\def\IC{\mathcal{I}}
\def\JC{\mathcal{J}}

\def\NC{\mathcal{N}}

\def\SC{\mathcal{S}}
\def\TC{\mathcal{T}}

\def\Pib{\mbox{\bm{$\Pi$}}}

\def\Real{\mbox{$\mathbb{R}$}}

\def\SymMat{\mbox{$\mathbb{S}$}}

\makeatletter
\def\widebar{\accentset{{\cc@style\underline{\mskip10mu}}}}
\def\Widebar{\accentset{{\cc@style\underline{\mskip8mu}}}}
\makeatother

\def\wt{\widetilde}

\def\OriProb{\mbox{$\mathsf{P}$}}
\def\RelaxProb{\mbox{$\mathsf{Q}$}}
\def\RelaxProbTrans{\mbox{$\mathsf{Q}_{\circ}$}}
\def\MVEEProb{\mbox{$\mathsf{R}$}}

\title{
Spectral Clustering by Ellipsoid and Its Connection to 
Separable Nonnegative Matrix Factorization
}

\author{Tomohiko Mizutani 
\thanks{
Department of Industrial Engineering and Management, 
Tokyo Institute of Technology,
2-12-1-W9-69, Ookayama, Meguro, Tokyo, 152-8552, Japan. 
{\tt mizutani.t.ab@m.titech.ac.jp}}}

\date{March 4, 2015}

\begin{document}

\maketitle

\begin{abstract}
 This paper proposes a variant of the normalized cut algorithm for spectral clustering.
 Although the normalized cut algorithm applies the K-means algorithm to the eigenvectors of 
 a normalized graph Laplacian for finding clusters,
 our algorithm instead uses a minimum volume enclosing ellipsoid for them.
 We show that the algorithm shares similarity with
 the ellipsoidal rounding algorithm for separable nonnegative matrix factorization.
 Our theoretical insight implies that 
 the algorithm can serve as a bridge between spectral clustering and separable NMF.
 The K-means algorithm has the issues in that 
 the choice of initial points affects the construction of clusters
 and certain choices result in poor clustering performance.
 The normalized cut algorithm inherits these issues since K-means is incorporated in it,
 whereas the algorithm proposed here does not.
 An empirical study is presented to examine the performance of the algorithm.
\end{abstract}

\section{Introduction}

Clustering is a task of dividing a data set into groups on the basis of 
similarities between pairs of data points. 
The task is to find groups of data points 
such that similar data points are in the same group and dissimilar ones in different groups. 
Here, the groups found by an algorithm for the clustering task are referred to as clusters. 
Spectral clustering is a graph-based clustering, and 
the eigenvalues and eigenvectors of the graph Laplacian play a central role.

In spectral clustering, we construct a weighted graph 
to represent the similarities of data points.
The vertices correspond to data points, and the edges are associated with weights.
The weights are determined by a similarity function 
that quantifies the similarity of two data points;
it takes on a large positive value if the data points are similar,
while it gets close to a zero value if they are dissimilar.
Small weights are usually ignored when constructing the graph
since they may not make a major contribution to the configuration of clusters.
For each data point, we pick up some data points with high similarity with it, and
put edges with positive weights between them.
An input parameter $p$, called the neighbor number, 
determines how many data points are chosen.

In the weighted graph, 
the clustering task is to divide the vertex set into groups such that 
the total edge weight in the same groups is large, 
while those among different groups are small.
Shi and Malik in \cite{Shi00} introduced a normalized cut function to formulate this task.
The function  assigns a nonnegative real number to the groups of vertices, 
and it reaches its minimum when the task is completed.
The ideal goal is to find groups of vertices that minimize a normalized cut function.

However, finding the optimal solution 
of the normalized cut minimization problem is hard.
We instead solve a relaxation problem formed by dropping hard constraints.
This is an eigenvalue problem for a normalized graph Laplacian.
The eigenvalues and eigenvectors of the Laplacian 
contain clues for finding the optimal solution of the normalized cut minimization problem.
Thus, we attempt to find clusters by using them.
The normalized cut algorithm proposed 
by Shi and Malik in \cite{Shi00} and  Ng et al.\ in \cite{Ng02}
applies a K-means algorithm to the eigenvectors.
We will use the abbreviation NC to denote this algorithm for short.

Instead of K-means in NC,
we propose to use the minimum volume enclosing ellipsoid for the eigenvectors 
of normalized graph Laplacian.
The computation of such an ellipsoid can be formulated as a convex optimization 
problem, and there are efficient algorithms for solving it.
Our algorithm computes the enclosing ellipsoid and chooses some points 
lying on the boundary by using the 
successive projection algorithm (SPA) of \cite{Gil14}.
The points can be thought of as representatives of the clusters.
Hence, our algorithm assigns data points to the representative points 
on the base of their contribution.
The assignment can be formulated as a convex optimization problem.
Figure~\ref{Fig: Exp1} of Section~\ref{Subsec: Illustration} 
illustrates the algorithm.

We see in Theorem \ref{Theo: Relation of NCER and ER} that 
the algorithm has a similarity to the ellipsoidal rounding (ER) algorithm 
for separable nonnegative matrix factorization (NMF) in \cite{Miz14}
when the neighbor number $p$ is set to be equal to the number of data points,
in other words, when all weights are taken into account in constructing the graph.
Strictly speaking, 
the final outputs returned by these two algorithms do not coincide.
However, we see in Corollary \ref{Coro: Relation of NCER and MER} that 
the  outputs coincide if we modify one step of ER.
Accordingly, 
our algorithm can be thought of as an extension of ER.
A separable NMF is a special case of NMF, 
and it has applications in clustering and topic extraction of documents \cite{Aro12b, Aro13, Miz14}
and in endmember detection of hyperspectral images \cite{Gil13, Gil14}.
It is a matrix factorization problem, and basically differs 
in purpose from spectral clustering.
However, the theoretical insights shown here imply that 
our algorithm can serve as a bridge between 
spectral clustering and separable NMF through neighbor number $p$.

The K-means algorithm has the issues in that the choice of initial points
is sensitive to the way clusters are constructed;
some choices yield good clustering performance, while others do not.
It is difficult to choose good initial points before running the algorithm 
and the choice affects the cluster construction.
Hence, NC inherits the issues.
There have been many studies indicating that 
NMF based clustering within the block coordinate descent (BCD) framework 
has a good performance. However, the framework has similar issues as K-means.
Our algorithm does not have these issues,
since it consists of solving an eigenvalue problem and convex optimization problems 
and performing SPA.

We conducted experiments evaluating the performance of our algorithm
on real image data sets and compared its results with those of existing algorithms, including NC and NMF.
We set multiple initial points for the existing algorithms
and measured the worst, best, and average performance.
The experimental results showed that 
the performance of our algorithm is higher than the average performance 
of the existing algorithms in almost all cases.
We observed that there are initial points 
that result in the existing algorithms having poor performance, 
and as a result, their average performance gets worse.
We also conducted experiments to see 
how the performance our algorithm varies with the neighbor number $p$.

The rest of this paper is organized as follows.
After introducing the notation and symbols,
we review the NC algorithm in Section~\ref{Sec: Review of NC}.
Our algorithm is presented in Section~\ref{Sec: Proposed algorithm}, and 
its connection with the ER algorithm is shown in Section~\ref{Sec: Connection to separable NMF}.
We mention the issues of initial point choice in NC and ER in Section~\ref{Sec: Issues},
and review related work in order to discuss the relationship between spectral clustering and NMF 
in Section~\ref{Sec: Related work}.
An empirical study is reported in Section~\ref{Sec: Experiments}.
Finally, concluding remarks are given in Section~\ref{Sec: Concluding remarks}.

\subsection{Notation and Symbols}
We use $\Real^{d \times m}$ and $\Real^{d \times m}_+$ to denote
the set of $d$-by-$m$ real matrices and $d$-by-$m$ nonnegative matrices.
Here, a nonnegative matrix is a real matrix whose elements are all nonnegative.
We also use $\SymMat^m$ to denote
the set of $m$-by-$m$ real symmetric matrices.
Let $\A$ be a matrix of proper size.
The symbol $\A^\top$ represents its transpose.
The symbols $\mbox{tr}(\A)$ and $\mbox{rank}(\A)$ represent the trace and rank.
We use $\e$ and $\e_i$
to denote a vector of all ones and an $i$th unit vector.
We use $\I$ to denote an identity matrix.
The symbol $\mbox{diag}(a_1, \ldots, a_m)$ represents an $m$-by-$m$ 
diagonal matrix such that diagonal elements are $a_1, \ldots, a_m$.
Let $\A_1$ and $\A_2$ be a $d$-by-$m_1$  matrix and a $d$-by-$m_2$ matrix.
We use $(\A_1, \A_2)$ to denote
the horizontal concatenation of the two matrices 
and the matrix size is $d$-by-$(m_1 + m_2)$.
For a set $\SC$,
the symbol $|\SC|$ represents the number of elements,
and the symbol $\SC^c$ is the complementary set.

\section{Review of Normalized Cut Algorithm for Spectral Clustering} 
\label{Sec: Review of NC}
We denote $m$ data points by 
$d$-dimensional vectors $\a_1, \ldots, \a_m$, and its set by $\SC$.
Consider $r$ subsets of $\SC$.
If the subsets are disjoint and its union coincides with $\SC$, 
we call the subsets {\it disjoint partitions} of $\SC$.
In this paper, we consider clustering algorithms to return the disjoint partitions of $\SC$,
and call the disjoint partitions returned by them {\it clusters} of data points.
Spectral clustering is a graph-based clustering, 
and the algorithm is based on the eigenvalue decomposition of a graph Laplacian.
There are some types of algorithms proposed for spectral clustering.
In particular,
the NC algorithm by Shi and Malik of \cite{Shi00} and Ng et al.\ of \cite{Ng02}
is popular and often used.
For the details of algorithms and history in spectral clustering,
we refer the reader to the survey paper \cite{Lux07}.
Below, we review the NC algorithm.

In spectral clustering, 
we set a function $f$ on a weighted graph $G$, and formulate 
a clustering task as a problem of minimizing $f$.
A weighted graph is a graph such that each edge is associated with a weight.
Let $V$ and $E$ denote the sets of the vertices and edges.
The weight value is given by a function $k$ from $V \times V$ 
to a set of nonnegative real numbers.
An edge $e \in E$ links two vertices $v_i$  and $v_j$ 
if the value of $k$ at $e$ is positive; otherwise, it does not link.
For an edge $e_{ij}$ between two vertices $v_i$ and $v_j$, 
let $k_{ij}$ denote the value of $k$ at $e_{ij}$.
Consider a subset $\SC$ of vertex set $V$.
We use the notation $\mbox{cut}(\SC, \SC^c)$ to 
denote the total weight of all edges between $\SC$ and its complement $\SC^c$.
Namely,
\begin{equation*}
 \mbox{cut}(\SC, \SC^c) = \frac{1}{2} \sum_{i \in \SC, j \in \SC^c} k_{ij}.
\end{equation*} 
In the same manner to \cite{Lux07}, 
hereinafter, we use a shorthand notation $i \in \SC$.
This notation represents the indices $i$ of vertices $v_i$ in $\SC$.
A {\it degree} of vertex $v_i$ is the total weight of all edges connected to $v_i$, 
and we denote it by $d_i$.
We use the notation $\mbox{vol}(\SC)$ 
to denote the total degree of all vertices in $\SC$.
Namely,
\begin{equation*}
 \mbox{vol}(\SC) = \sum_{i \in \SC} d_{i}
\end{equation*}
where $d_i$ is a degree of vertex $v_i$ in $\SC$ and 
it is given as $d_i = \sum_{j=1}^{m} k_{ij}$ 
for the weights $k_{ij}$ on edges $e_{ij}$.
The $\mbox{vol}(\SC)$ can be regarded as the size of $\SC$.

The NC algorithm consists of three major steps.
The first step constructs a weighted graph $G$.
A vertex $v_i \in V$ is in one-to-one correspondence with 
a data point $\a_i \in \SC$.
Let $k$ be a function that quantifies the similarity of two data points.
The function assigns a nonnegative real number 
as the similarity of data points $\a_i$ and $\a_j$;
it takes on a large positive value if the data points are similar,
while it gets close to a zero value if they are dissimilar.
In the context of spectral clustering, the function is referred to as a {\it similarity function}.
A polynomial function and a Gaussian function are popular and often used as the function.
In particular, this paper deals with the former function, and its form is 
\begin{equation} \label{Eq: Poly}
 k(\a_i, \a_j) = (\a_i^\top \a_j + b)^c
\end{equation}
where $b$ is a nonnegative real number 
and $c$ is a positive integer number.
These are parameters given in advance.

Small weight values in $G$ may not make a major contribution to 
the configuration of clusters, and thus, are usually ignored.
A $p$-nearest neighbor set is used for this purpose.
For a data point $\a_i$, 
we choose the top $p$ data points with high similarity to $\a_i$
as measured by a similarity function $k$, and construct a set 
by collecting these points.
Let $\NC_p(\a_i)$ denote the set.
The integer number $p$ used for the construction is
referred to as a {\it neighbor number}.
The weight value $k_{ij}$ of $G$ is given as 
\begin{equation*}
 k_{ij} = \left\{
 \begin{array}{ll}
  k(\a_i, \a_j), & \mbox{if} \ \a_i \in \NC_p(\a_j) \ \mbox{or} \ \a_j \in \NC_p(\a_i), \\
  0, & \mbox{otherwise}.
 \end{array}
\right.
\end{equation*}
We denote the $m$-by-$m$ symmetric matrix consisting of $k_{ij}$ by $\K$. 
The matrix $\K$ is called a {\it weighted adjacency matrix} of  $G$, 
and in this paper, it is referred to as an {\it adjacency matrix} for short.

The next step finds clusters from the graph $G$ built in the first step. 
A graph Laplacian is a matrix so as to possess some properties of $G$.
To see the form of this matrix, we need to introduce a degree matrix.
A {\it degree matrix} of $G$ is a diagonal matrix
whose diagonal element is the degree of each vertex.
Namely, the $(i,i)$th element is given as $d_i = \sum_{j=1}^{m} k_{ij}$ 
where $k_{ij}$ is an element of the adjacency matrix $\K$ of $G$. 
We denote the $m$-by-$m$ diagonal matrix by $\D$. 
A degree matrix is usually assumed to be nonsingular.
The singularity means that 
some data point is isolated and is completely dissimilar to all the other data points. 
Thus, such a data point should be removed.
A {\it graph Laplacian} of $G$ is a matrix given as $\D-\K$ 
for a degree matrix $\D$ and an adjacency matrix $\K$.
This is an $m$-by-$m$ symmetric matrix, and we denote it by $\L$.

We set a function $f$ to find the clusters from $G$.
Although some choices are possible, we consider the function 
for the disjoint partitions $\SC_1, \ldots, \SC_r$ of $\SC$,
\begin{equation} \label{Eq: Normized Cut Function}
 f(\SC_1, \ldots, \SC_r) 
  = \sum_{i=1}^r \frac{\mbox{cut}(\SC_i, \SC_i^c)}{\mbox{vol}(\SC_i)}.
\end{equation}
This function is called a {\it normalized cut}.
Other types of functions have been proposed in this context. 
For instance, a ratio cut function \cite{Hag92} is as popular as the normalized cut one.
The $\mbox{cut}(\SC_i, \SC_i^c)$ takes a small value
if the data points in $\SC$ and $\SC^c$ are dissimilar.
The $\mbox{vol}(\SC_i)$ takes a large value if so is the size of $\SC$.
Therefore, the minimization of the normalized cut function 
is to find the clusters such that 
the data points in different clusters are dissimilar and the cluster sizes are large.

The ideal goal is to find $r$ disjoint partitions of $\SC$ minimizing the normalized cut function.
To be precise, 
\begin{quote}
 {\bf (Spectral clustering by the normalized cut function)} \
 Suppose that we are given a data set $\SC$ and an integer number $r$.
 Find the disjoint partitions 
 $\SC_1, \ldots, \SC_r$ of $\SC$ to minimize the normalized cut function $f$
 of (\ref{Eq: Normized Cut Function}).
\end{quote}
The optimal solution of the minimization problem provides the best clusters in the sense that
it attains the minimum of $f$.
But, this is a hard combinatorial problem, and thus we consider its relaxation problem.
By using a graph Laplacian $\L$,
we rewrite $f$ as
\begin{equation*}
 f = \mbox{tr}(\H^\top \L \H).
\end{equation*}
Here,  $\H$ is an $m$-by-$r$ matrix determined by $\SC_1, \ldots, \SC_r$, 
and its element $h_{ij}$ is 
\begin{equation}
 \label{Eq: Element of H in ncut}
 h_{ij} = 
  \left\{
  \begin{array}{ll}
   1  /  \sqrt{\mbox{vol}(\SC_j)}, & \a_i \in \SC_j, \\
   0,                         & \mbox{otherwise}.
  \end{array}
  \right.
\end{equation}
Therefore, the minimization problem of $f$ is equivalent to the problem,
\begin{equation*}
 \OriProb :
   \begin{array}{cccc}
    \mbox{minimize} & \mbox{tr}(\H^\top \L \H)  & \mbox{subject to} 
     &  \H \ \mbox{satisfies} \ (\ref{Eq: Element of H in ncut}).
   \end{array}
\end{equation*}
The hardness of this problem is in the constraint for $\H$. 
Hence, we drop the hard constraint.
Instead, we take into account that $\H$ satisfies the relation $\H^\top \D \H=\I$ 
and add it as the constraint.
Namely, we consider the problem,
\begin{equation*}
 \RelaxProb :
  \begin{array}{cccc}
   \mbox{minimize} & \mbox{tr}(\H^\top \L \H)  & \mbox{subject to}  &  \H^\top\D\H = \I.
  \end{array}
\end{equation*}
The problem $\RelaxProb$ serves as the relaxation problem of $\OriProb$.
It can be solved through eigenvalue decomposition.
By introducing a new matrix variable $\G \in \Real^{m \times r}$ such that $\G = \D^{1/2}\H$,
we transform the problem into an equivalent one,
\begin{equation*}
 \RelaxProbTrans :
  \begin{array}{cccc}
   \mbox{minimize} & \mbox{tr}(\G^\top \D^{-1/2}  \L \D^{-1/2} \G)  & \mbox{subject to}  &  \G^\top \G = \I.
  \end{array}
\end{equation*}
In this transformation, we use the assumption that 
a degree matrix $\D$ is nonsingular, and in other words, 
the diagonal elements are all positive.
The optimal value and solution of $\RelaxProbTrans$ is obtained 
from the eigenvalue decomposition of normalized graph Laplacian $\D^{-1/2}  \L \D^{-1/2}$. 
We arrange the eigenvalues in ascending order, and 
let $\lambda_1, \ldots, \lambda_m$ denote the values.
Namely, the relation $\lambda_1 \le \cdots \le \lambda_m$ holds.
Also, let $\v_i$ denotes the eigenvector corresponding to the eigenvalue $\lambda_i$
for $i = 1, \ldots, m$.
We easily see that the optimal value and optimal solution of $\RelaxProbTrans$ 
are given as $\lambda$ and $\V_r$ such that $\lambda = \lambda_1 + \cdots + \lambda_r$ 
and $\V_r = (\v_1, \ldots, \v_r)$. 
Therefore, those of $\RelaxProb$ are respectively $\lambda$ and $\D^{-1/2}\V_r$.

The final step constructs clusters from the optimal solution $\D^{-1/2}\V_r$ 
of relaxation problem $\RelaxProb$. 
We search for the optimal solution $\H$ of original problem $\OriProb$
based on clues provided by  that of $\RelaxProb$.
Let us see the matrix $\H$ of $(\ref{Eq: Element of H in ncut})$ in detail.
We take the transpose of $\H$, and denote 
the column vectors of $\H^\top$ by $\f_1, \ldots, \f_m \in \Real^r$.
The vector $\f_i$ can be regarded as an indicator to 
tell us which cluster a data point belongs to.
The elements of $\f_i$ are all zero except one element, and 
the position of nonzero element indicates the cluster index 
to which a data point belongs.
The convex hull of $\f_1, \ldots, \f_m$
is an $(r-1)$-dimensional simplex in $\Real^r$.
Among $\f_1, \ldots, \f_m$, there are $r$ different types of vectors and 
those different ones correspond to the $r$ vertices.

Let us consider the situation in which the optimal solution $\D^{-1/2}\V_r$ 
of relaxation problem $\RelaxProb$ is close to the optimal solution $\H$ 
of original one $\OriProb$.
In the same way as $\f_i$, 
we take the transpose of $\D^{-1/2}\V_r$ and denote the column vectors of $\V_r^{\top}\D^{1/2}$
by $\p_1, \ldots, \p_m \in \Real^r$.
Under this situation, we can have an expectation that 
the convex hull of $\p_1, \ldots, \p_m$ is similar to the shape 
of an $(r-1)$-dimensional simplex,
and $\p_1, \ldots, \p_m$ are located around each vertex of the simplex.
Thus, these vectors should form $r$ clusters.
Accordingly, the clusters can be found by applying a clustering algorithm
such as K-means to $\p_1, \ldots, \p_m$.
Algorithm~\ref{Alg: NC} describes each step of the algorithm by Shi and Malik in \cite{Shi00}.

\begin{algorithm}
 \caption{NC \cite{Shi00}}
 \label{Alg: NC}
 \textbf{Input:}  
 A data set $\SC = \{\a_1, \ldots, \a_m\}$, a cluster number $r$,  a neighbor number $p$, 
 and a similarity function $k$. \\
 \textbf{Output:} Clusters $\SC_1, \ldots, \SC_r$.
 \begin{enumerate}
  \item[\textbf{1:}] 
	       Construct an adjacency matrix $\K \in \SymMat^m$ by using 
	       a similarity function $k$ and a $p$-nearest neighbor set.

  \item[\textbf{2:}] 
	       Let $\L$ and $\D$ be a graph Laplacian and 
	       a degree matrix obtained from $\K$.
	       Compute the $r$ eigenvectors $\v_1, \ldots, \v_r \in \Real^m$ 
	       of the normalized graph Laplacian $\D^{-1/2} \L \D^{-1/2} \in \SymMat^m$
	       in correspondence with the $r$ smallest eigenvalues.

  \item[\textbf{3:}] 
	       Form a matrix $\V_r = (\v_1, \ldots, \v_r)$, and 
	       let $\p_1, \ldots, \p_m \in \Real^r$ be 
	       the columns of $\V_r^{\top}\D^{-1/2}$.

  \item[\textbf{4:}] 
	       Apply the K-means algorithm to 
	       $\p_1, \ldots, \p_m$, and find 
	       $r$ clusters $\SC_1, \ldots, \SC_r$.
 \end{enumerate}
\end{algorithm}

\begin{remark} \label{Remark:Alg-by-Ng}
 The algorithm by Ng et al.\ of \cite{Ng02} coincides with that of Shi and Malik of \cite{Shi00}
 except for $\D^{-1/2}$ in the step 3 of Algorithm~\ref{Alg: NC}.
 Instead of the matrix, it uses a diagonal matrix $\S$ that 
 scales each column of $\V^\top$ to have a unit $\ell_2$ norm.
\end{remark}

\section{Proposed Algorithm}
\label{Sec: Proposed algorithm}
In this section, we will describe the proposed algorithm.
The algorithm is a variant of the NC algorithm.
Although NC uses K-means for the construction of clusters in the final step,
it instead uses a minimum volume enclosing ellipsoid (MVEE).

\subsection{Observation for the Optimal Solution of $\RelaxProb$}
\label{Subsec: Observation}
Our algorithm is built on the following observation. 
\begin{observation} \label{Obs: Dist of eigenvec}
 Let $\p_1, \ldots, \p_m \in \Real^r$ denote the column vectors 
 of $\V_r^{\top} \D^{-1/2} \in \Real^{r \times m}$.
 The convex hull of $\p_1, \ldots, \p_m$ has 
 a similar shape to an $(r-1)$-dimensional simplex in $\Real^r$, 
 and $\p_1, \ldots, \p_m$ are around the $r$ vertices.
 Therefore, these vectors form $r$ clusters in the neighborhood of the vertices.
\end{observation}
As mentioned in the last of Section \ref{Sec: Review of NC},
this observation should hold if 
$\D^{-1/2}\V_r$ is close to the optimal solution $\H$ of $\OriProb$.
We see in Proposition \ref{Prop: Vecotrs-lie-on-hyperplane} 
that the $\p_1, \ldots, \p_m \in \Real^r$ lie on a hyperplane, and 
thus, its convex hull is at least an $(r-1)$-dimensional polytope.
In the following description, we use the same notation in Section \ref{Sec: Review of NC};
$\v_1, \ldots, \v_r \in \Real^m$ are the eigenvectors of 
normalized graph Laplacian $\D^{-1/2}\L\D^{-1/2} \in \SymMat^m$,
and those correspond to each of the eigenvalues $\sigma_1, \ldots, \sigma_r$
with $\sigma_1 \le \cdots \le \sigma_r$;
$\p_1, \ldots, \p_m \in \Real^r$ are the column vectors 
of $\V_r^{\top}\D^{-1/2} \in \Real^{r \times m}$ 
with $\V_r = (\v_1, \ldots, \v_r)$ given by the $\v_1, \ldots, \v_r$.

\begin{prop} \label{Prop: Vecotrs-lie-on-hyperplane}
 $\v_1$ can be chosen as $\tau\D^{1/2}\e$ with some nonzero real number $\tau$.
 Thus, if we choose $\v_1$ as $\tau \D^{1/2}\e$,
 then, $\p_1, \ldots, \p_m$ lie on a hyperplane $\HC = \{\x \in \Real^r : \e_{1}^\top \x = \tau \}$.
\end{prop}
\begin{proof}
 It suffices to show the first part of this statement that 
 $\v_1$ can be chosen as $\tau\D^{1/2}\e$.
 Let $\widebar{\L}$ denote 
 a normalized graph Laplacian $\D^{-1/2}\L\D^{-1/2}$.
 Since a graph Laplacian $\L$ is positive semidefinite, so is $\widebar{\L}$.
 Also, we have $\L\e = \0$.
 Thus, $\widebar{\L}$ has a zero eigenvalue as the smallest one, 
 and the corresponding eigenspace contains a vector $\tau\D^{1/2}\e$ 
 with some nonzero real number $\tau$.
\end{proof}

Regarding the first part of the proposition,
we can find the same statement in Proposition 3 of \cite{Lux07}.

\begin{remark}
 Although Proposition \ref{Prop: Vecotrs-lie-on-hyperplane} requires us to
 choose the vector $\tau \D^{1/2}\e$ from the eigenspace 
 associated with the smallest eigenvalue,
 eigenvalue solvers such as {\tt eig} and {\tt eigs} commands on MATLAB
 do not always output the vector.
 But, we, of course, can reconstruct it from the eigenvectors provided by the solvers.
 Let $\EC$ denote the eigenspace in $\Real^m$ associated with the smallest eigenvalue.
 Suppose that the dimension of $\EC$ is $s$. 
 Since a normalized graph Laplacian is a symmetric matrix,
 we can take $s$ orthonormal basis vectors for $\EC$.
 Let $\v_1, \ldots, \v_s \in \Real^m$ denote the $s$ orthonormal basis vectors.
 Then, there exists a nonzero  vector $\p \in \Real^s$ 
 such that $\tau \D^{1/2}\e = \V_s\p$ and $||\p||_2=1$ 
 where $\V_s = (\v_1, \ldots, \v_s) \in \Real^{m \times s}$ since $\tau \D^{1/2}\e \in \EC$. 
 We pick up $s-1$ orthonormal basis vectors $\p_1, \ldots, \p_{s-1} \in \Real^s$
 for the orthogonal complement to $\p$. 
 Let $\P = (\p, \p_1, \ldots, \p_{s-1}) \in \Real^{s \times s}$.
 We construct $\V_s\P \in \Real^{m \times s}$.
 Since $\P$ is an orthogonal matrix, the $s$ column vectors of $\V_s \P$ are 
 the orthonormal basis vectors spanning $\EC$, and 
 the first column vector is $\tau \D^{1/2}\e$. 
\end{remark}

\begin{remark}
 The matrix $\D^{-1/2}$ in the step 3 of Algorithm~\ref{Alg: NC}
 can be thought of as a scaling matrix that scales the column vectors of $\V_r^\top$ 
 to lie on a hyperplane.
 As mentioned in Remark~\ref{Remark:Alg-by-Ng}, instead of $\D^{-1/2}$,
 the algorithm by Ng et al.\ of \cite{Ng02} uses a diagonal matrix $\S$
 that scales those of $\V_r^\top$ to lie on a unit cube.
\end{remark}

\subsection{MVEE for the Points in Simplex} \label{Subsec: MVEE for the points in Simplex}
Observation~\ref{Obs: Dist of eigenvec} implies that 
$\p_1, \ldots, \p_m \in \Real^r$ form $r$ clusters in the neighborhood of the vertices of 
an $(r-1)$-dimensional simplex in $\Real^r$.
Thus, if we can find $r$ vectors from $\p_1, \ldots, \p_m$ that 
are close to the vertices of the simplex,
those vectors should serve as the representative points of clusters.
An MVEE for $\p_1, \ldots, \p_m$ can be used for finding such vectors. 
This is because it can touch the near-vertices if 
the convex hull of $\p_1, \ldots, \p_m$ is similar to the shape of simplex.
Originally, this geometric property has been shown in \cite{Gil13, Miz14} 
in order to design algorithms for a separable NMF problem.

In this section, we first see the formulation of MVEE, 
and then recall the precise description of the property 
in Propositions \ref{Prop: Number of touching point} and \ref{Prop: Vertices found by MVEE}.
Let $\L$ be an $r$-by-$r$ positive definite matrix, and $\z$ be an $r$-dimensional vector.
A full-dimensional ellipsoid in $\Real^r$ is defined as a set 
$\EC = \{ \x \in \Real^r : (\x - \z)^{\top}\L (\x - \z) \le 1\}$.
The vector $\z$ serves as the center of ellipsoid $\EC$.
In particular, an ellipsoid is referred to as an origin-centered ellipsoid 
if the center $\z$ matches the origin (in other words, $\z$ is a zero vector).
It is known that the volume of unit ball in $\Real^r$ only depends on the dimension $r$,
and we denote it by $c(r)$.
The volume of $\EC$ is given as $c(r) / \sqrt{\det \L}$.
Let $\BC$ denote the boundary of $\EC$ such that 
$\BC = \{ \x \in \Real^r : (\x - \z)^{\top}\L (\x - \z) = 1\}$.
A vector $\x$ is called as an {\it active point} of $\EC$
if it lies on the boundary such that $\x \in \BC$.

Let $\p_1, \ldots, \p_m$ be $m$ points in $\Real^r$.
We put the following assumption on the points.

\begin{assump} \label{Assump: Full row rank}
$\mbox{rank}(\P) = r$ for $\P = (\p_1, \ldots, \p_m) \in \Real^{r \times m}$.
\end{assump}

We consider an origin-centered ellipsoid to enclose the convex hull of
a set $\SC = \{\pm \p_1, \ldots, \pm \p_m\}$, and 
describe several properties of the enclosing ellipsoid with minimum volume.
Assumption \ref{Assump: Full row rank} ensures that the convex hull of $\SC$ 
is full-dimensional in $\Real^r$, and thus, the enclosing ellipsoid has a positive volume.
The volume minimization of enclosing ellipsoid for the convex hull of $\SC$ 
is formulated as the problem,
\begin{equation*} 
 \begin{array}{lll}
  \MVEEProb(\SC): & \mbox{minimize}    & -\log \det \L, \\
                  & \mbox{subject to}  & \langle  \p \p^{\top}, \L \rangle \le 1
		   \ \mbox{for all} \ \p \in \SC, \\
                  &                    & \L \succ \0.
 \end{array}
\end{equation*}
The $r$-by-$r$ matrix $\L$ is the decision matrix variable.
The symbol $\succ$ for a matrix represents that it is positive definite.
In the problem $\MVEEProb(\SC)$, we need to put Assumption \ref{Assump: Full row rank} 
to ensure the existence of optimal solution.
Let $\L^*$ denote the optimal solution.
The set $\{\x \in \Real^r : \x^\top \L^* \x \le 1 \}$ 
is an origin-centered enclosing ellipsoid with minimum volume for the convex hull of $\SC$. 
This paper refers to it as an {\it origin-centered MVEE} for $\SC$. 
The active points of the origin-centered MVEE
are the points $\p_i$ satisfying $\p_i^\top \L^* \p_i = 1$.
The problem $\MVEEProb(\SC)$ is a convex optimization problem, 
and there are polynomial-time algorithms for solving it.
The problem and algorithms for computing an MVEE  have been well studied, and 
we refer the reader to, for instance, \cite{Boy04} for the details.

We describe the propositions mentioned in the first of this section.

\begin{prop}[Lemma 11 of \cite{Miz14}] \label{Prop: Number of touching point}
 Suppose that $\p_1, \ldots, \p_m \in \Real^r$ satisfy Assumption \ref{Assump: Full row rank}.
 Then, the origin-centered MVEE for $\SC = \{ \pm \p_1, \cdots, \pm \p_m\}$ touches at least  $r$ points 
 with plus-minus signs among $\p_1, \ldots, \p_m$.
\end{prop}

\begin{prop}[Proposition 3 and Corollary 4 of \cite{Miz14}] \label{Prop: Vertices found by MVEE}
 Consider an $(r-1)$-dimensional simplex in $\Real^r$.
 Let $\p_1, \ldots, \p_r \in \Real^r$ be the vertices, 
 and let $\P = (\p_1, \ldots, \p_r) \in \Real^{r \times r}$.
 \begin{enumerate}[{\normalfont (a)}]
  \item 
	Suppose that $\q_1, \ldots, \q_n \in \Real^r$ belong to the simplex.
	Construct a set 
	$\SC = \{\pm \p_1, \ldots, \pm \p_r, \pm \q_1, \ldots, \pm \q_n\}$.
	Then, the origin-centered MVEE for $\SC$ only touches the vertices $\p_1, \ldots, \p_r$ 
	with plus-minus signs.

  \item 
	Suppose that $\q_1, \ldots, \q_n \in \Real^r$ are given as
	$\q_i = \P\k_i$ by using $\k_i \in \Real^r$ such that $||\k_i||_2 < 1$.
	Construct a set 
	$\SC = \{\p_1, \ldots, \pm \p_r, \pm \q_1, \ldots, \pm \q_n \}$.
	Then, the origin-centered MVEE for $\SC$ only touches the vertices $\p_1, \ldots, \p_r$ of 
	with plus-minus signs.
 \end{enumerate}
\end{prop}

The proof for Proposition \ref{Prop: Vertices found by MVEE}(a) is also found in  \cite{Gil13}.
Proposition \ref{Prop: Vertices found by MVEE}(b)
can be thought of an extension of that of Proposition \ref{Prop: Vertices found by MVEE}(a) 
in some sense since $\q_1, \ldots, \q_n$ in Proposition \ref{Prop: Vertices found by MVEE}(a) 
can be written as $\q_i = \P\k_i$ by using $\k_i$ such that $||\k_i||_1 = 1$ and $\k_i \ge \0$.
It should be noted that Assumption \ref{Assump: Full row rank} is satisfied
in Proposition \ref{Prop: Vertices found by MVEE}
since $\p_1,\ldots,\p_r$ are the vertices of $(r-1)$-dimensional simplex.

We see from Proportions \ref{Prop: Number of touching point} and \ref{Prop: Vertices found by MVEE}
that the following geometric properties hold.
Consider a finite number of points, including the $r$ vertices, 
in an $(r-1)$-dimensional simplex in $\Real^r$.
Then, the origin-centered MVEE for the points touches the $r$ vertices. 
Furthermore, 
it also holds even if some amount of perturbation is added to the points.
If the perturbation is large, the ellipsoid can touch more that $r$ points.

\subsection{Algorithm Description} \label{Subsec: Algorithm description}

Recall Observation~\ref{Obs: Dist of eigenvec};
$\p_1, \ldots, \p_m \in \Real^r$
are the column vectors of $\V_r^\top \D^{-1/2} \in \Real^{r \times m}$ 
whose transpose is the optimal solution of problem $\RelaxProb$.
Suppose that Observation~\ref{Obs: Dist of eigenvec} holds.
Then, the convex hull of $\p_1, \ldots, \p_m$ is similar to the shape of a simplex.
Proposition \ref{Prop: Vertices found by MVEE}
tells us that the MVEE for the set $\SC = \{\pm \p_1, \ldots,  \pm \p_m\}$
can touch some vectors among $\p_1, \ldots, \p_m$ that are 
close to the vertices of the simplex.
Such vectors should serve as the representative points of clusters.
Accordingly, clusters can be found by assigning $\p_1, \ldots, \p_m$ 
to the representative points on the base of their contribution.
Here, we should pay attention to the following issue.
Proposition \ref{Prop: Number of touching point} implies that 
the MVEE for the $\SC$ has a possibility to touch more than $r$ vectors
among $\p_1, \ldots, \p_m$.
The algorithm for a separable NMF problem such as SPA \cite{Gil14} can be used 
for selecting $r$ points from the candidates.
We will explain the problem and algorithms in Section \ref{Subsec: Separable NMF problem}.

\begin{algorithm}
 \caption{NCER}
 \label{Alg: NCER}
 \textbf{Input:}  
 A data set $\SC = \{\a_1, \ldots, \a_m\}$, a cluster number $r$,  
 a neighbor number $p$, and a similarity function $k$. \\
 \textbf{Output:} Clusters $\SC_1, \ldots, \SC_r$.
 \begin{enumerate}
  \item[\textbf{1:}] 
	       Run steps 1-3 of Algorithm \ref{Alg: NC}, 
	       and let $\SC = \{\pm \p_1, \ldots, \pm \p_m\}$
	       where these vectors $\p_1, \ldots, \p_m \in \Real^r$ are the columns
	       of $\V_r^\top\D^{-1/2} \in \Real^{r \times m}$.

  \item[\textbf{2:}] 
	       Compute an origin-centered MVEE for the set $\SC$, 
	       and find all active points.
	       Let $\IC$ be the index set of the active points.

  \item[\textbf{3:}] 
	       If $|\IC| = r$, set $\JC$ as $\JC = \IC$.
	       Otherwise, select $r$ elements from $\IC$ by SPA, 
	       and construct the set $\JC$ of these $r$ elements.

  \item[\textbf{4:}] 
	       Assign each $\p_1, \ldots, \p_m$ to any one of 
	       $\p_j, \ j \in \JC$ on the base of contribution,
	       and construct $r$ clusters $\SC_1, \ldots, \SC_r$.
	       
 \end{enumerate}
\end{algorithm}

Our algorithm is presented in Algorithm \ref{Alg: NCER}.
We denote it by {\it NCER} since it is a variant of NC 
and uses an ellipsoid rounding to find clusters in the final step. 
Below, we explain the details of steps 3 and 4.

We see from Proposition \ref{Prop: Number of touching point} that 
the index set of active points $\IC$ constructed in the step 2
contains at least $r$ elements
since the vectors $\p_1, \ldots, \p_m \in \Real^r$ satisfy $\mbox{rank}(\P) = r$ 
for $\P = (\p_1, \ldots, \p_m) \in \Real^{r \times m}$.
Therefore, step 3 constructs a set $\JC$ by setting $\JC = \IC$ if $|\IC| = r$;
otherwise, selects $r$ elements from $\IC$, and constructs a set $\JC$.
SPA can be used for the selection.
If the convex hull of vectors $\p_i, \ i \in \IC$ is similar to 
the shape of a simplex, 
SPA can select $r$ vectors from $\p_i, \ i \in \IC$ that are close 
to the vertices of the simplex.
We refer the reader to Algorithm 1 of \cite{Gil14} for the details of SPA.

Step 4 assigns $\p_1, \ldots, \p_m$ to $r$ representative points $\p_j, \ j \in \JC$, and 
constructs $r$ clusters.
The assignment is conducted on the base of contribution rate of representative points 
in generating a point.
Let $\P(\JC)$ denote an $r \times r$ matrix consisting of 
the representative points $\p_j, \ j \in \JC$.
Namely, $\P(\JC) = (\p_j : j \in \JC)$.
Consider some $\p_i, \ i \in \{1, \ldots, m\}$. 
We define a contribution rate $\w$ of the representative points to $\p_i$ 
as the optimal solution of the problem,
\begin{equation*}
 \begin{array}{rrrr}
  \mbox{minimize} & ||\P(\JC)\w - \p_i||_2^2 & \mbox{subject to} & \w \ge \0.
 \end{array}
\end{equation*}
The $r$-dimensional vector $\w$ is the decision variable.
This is a convex optimization problem, and sometimes referred to as a 
{\it nonnegative least square (NLS) problem}.
The optimal solution can be obtained by using an active set algorithm.
We refer the reader to \cite{Law87} for the details of the algorithm.
Let $\w^*$ be the optimal solution of the problem.
We find the index $j$ of the largest element among the $r$ elements of $\w^*$, and 
assign $\p_i$ to a cluster $\SC_{j}$.

\section{Connection to Separable NMF}
\label{Sec: Connection to separable NMF}
In this section, 
we will see that NCER for spectral clustering 
is closely related to ER for separable NMF in \cite{Miz14}.
In fact, NCER shares similarity with ER when a neighbor number $p$ 
is set to be equal to the number of data points.

\subsection{Separable NMF Problem} \label{Subsec: Separable NMF problem}
Consider a $d$-by-$m$ nonnegative matrix $\A$ such that 
\begin{equation}
 \A = \F \W \ \mbox{for} \ \F \in \Real^{d \times r}_+ \ 
  \mbox{and} \ \W = (\I, \K) \Pib \in \Real^{r \times m}_+.
  \label{Eq: Separable matrix}
\end{equation}
Here, $\I$ is an $r$-by-$r$ identity matrix, $\K$ is an $r$-by-$(m-r)$ nonnegative matrix, and
$\Pib$ is an $m$-by-$m$ permutation matrix.
A nonnegative matrix $\A$ is called a {\it separable matrix} if it can be decomposed into 
$\F$ and $\W$ of the form (\ref{Eq: Separable matrix}).
We call the $\F$ a {\it basis matrix} and the $\W$ a {\it weight matrix}.
A separable NMF problem is a problem of finding 
the basis and weight matrices from a separable matrix.
To be precise, 
\begin{quote}
 {\bf (Separable NMF problem)} \
 Suppose that we are given a separable matrix $\A$ of the form (\ref{Eq: Separable matrix}) 
 and an integer number $r$.
 Find an index set $\IC$ with $r$ elements such that $\F = \A(\IC)$. 
\end{quote}
Here, $\A(\IC)$ denotes the submatrix of $\A$ which consists of the column vectors 
with indices in $\IC$.
Namely,  $\A(\IC) = (\a_i : i \in \IC)$ for the column vectors $\a_i$ of $\A$.

A separable NMF is the special case of NMF.
As we will mention in Section \ref{Sec: Issues},
solving an NMF problem is hard in general.
As a remedy for the hardness,
Arora et al.\ in \cite{Aro12a, Aro12b} proposed to make the assumption called separability on the NMF problem. 
Then, it turns into a tractable problem under the assumption.
An NMF problem under a separability assumption is referred to as a separable NMF problem.
Separable NMF has applications in clustering and topic extraction of documents 
\cite{Aro12a, Aro12b, Kum13} and endmember detection of hyperspectral images \cite{Gil13, Gil14}.

We shall look at the separable NMF problem from geometric  point of view.
Since the data matrix $\A$ is a separable one of the form (\ref{Eq: Separable matrix}),
the conical hull of the column vectors of $\A$ is an $r$-dimensional cone in $\Real^d$.
It has $r$ extreme rays, and those correspond to the column vectors of basis matrix $\F$.
The intersection of the cone with a hyperplane is an $(r-1)$-dimensional simplex in $\Real^d$,
and the $r$ vertices correspond to the column vectors of basis matrix $\F$.
Figure~\ref{Fig: geoSepMat} illustrates the geometric interpretation of separable matrix.
Hence, a separable NMF problem can be rewritten as a problem of finding all vertices 
of $(r-1)$-dimensional simplex in $\Real^d$.
Several types of algorithms have been proposed for a separable NMF problem.
SPA\cite{Gil14} and XRAY\cite{Kum13} as well as ER
are designed on the geometric interpretation of separable matrix.

\begin{figure}[h]
\begin{center}
 \includegraphics[width=0.4\linewidth]{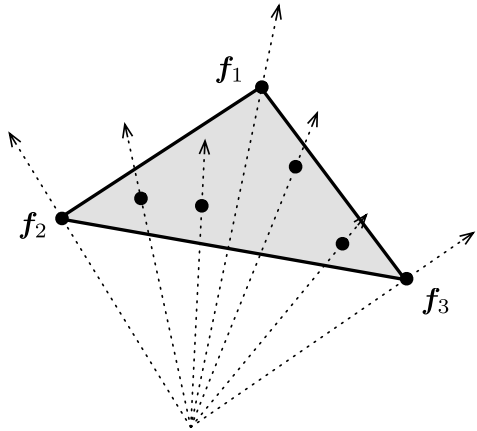}
 \caption{Geometric interpretation of a separable matrix $\A$ with $(d,m,r) = (3,7,3)$.
 Black points are the intersection points of
 the column vectors of $\A$ with a hyperplane.
 The region surrounded by the black lines represents the convex hull of intersection points.
 $\f_1, \f_2$, and $\f_3$ are the column vectors of basis matrix $\F$ in $\A$.}
 \label{Fig: geoSepMat}
\end{center}
\end{figure}

It is ideal that an algorithm for a separable NMF problem is robust 
to perturbation which disturbs the separability structure.
For a separable matrix $\A$ of (\ref{Eq: Separable matrix}) and a $d$-by-$m$ real matrix $\N$,
we consider the matrix
\begin{equation}
 \wt{\A}= \A + \N.
  \label{Eq: Near-separable matrix}
\end{equation}
We call the $\wt{\A}$ a {\it near-separable matrix}
since the separability structure of $\A$ is disturbed by $\N$.
A separable NMF algorithm is said to be {\it robustness} 
if the algorithm can find a matrix close to 
the basis matrix in a near-separable matrix.
It has been shown theoretically and practically 
in \cite{Miz14} and \cite{Gil14} that ER and SPA are robust algorithms.

\subsection{ER Algorithm} \label{Subsec: ER algorithm}
We give a precise description of ER.
As already mentioned, 
the algorithm is based on the geometric interpretation of separable matrix, and 
Propositions \ref{Prop: Number of touching point} and \ref{Prop: Vertices found by MVEE} 
are the backbone of the algorithm.
In the application of Proposition \ref{Prop: Vertices found by MVEE} to a separable NMF problem, 
we need to pay attention to the following point.
The proposition requires that an $(r-1)$-dimensional simplex is in an $r$-dimensional real space.
But, in the problem, it can be in higher dimensional space.
Therefore, a dimension reduction is required to be performed.

\begin{algorithm}
 \caption{ER}
 \label{Alg: ER}
 \textbf{Input:}  A $d$-by-$m$ nonnegative matrix $\A$, and a basis number $r$. \\
 \textbf{Output:} An index set $\JC$.
 \begin{enumerate}
  \item[\textbf{1:}] 
	       Perform a dimension reduction of $\A \in \Real^{d \times m}_+$ 
	       from $d$ to $r$ and construct a dimensionally reduced matrix 
	       $\B \in \Real^{r \times m}$ through SVD.

  \item[\textbf{2:}] 
	       Remove all of zero column vectors in $\B$ if exist.
	       Construct a diagonal matrix $\S \in \Real^{m \times m}$ such that 
	       all column vectors of $\B\S$ lie on a hyperplane $\HC$.
	       Let $\SC = \{\pm \q_1, \ldots, \pm \q_m\}$ 
	       for the column vectors  $\q_1, \ldots, \q_m \in \Real^r$ of $\B\S$.

  \item[\textbf{3:}] 
	       Compute an origin-centered MVEE for the set $\SC$, 
	       and find all active points.
	       Let $\IC$ be the index set of the active points.

  \item[\textbf{4:}] 
	       If $|\IC| = r$, set $\JC$ as $\JC = \IC$.
	       Otherwise, select $r$ elements from $\IC$ by SPA, 
	       and construct the set $\JC$ of these elements.

 \end{enumerate}
\end{algorithm}

Algorithm~\ref{Alg: ER} describes each step of ER.
In the description, 
the notation $m$ is still used to denote the number of column vectors of $\B$ 
for simplicity
although it varies if the zero column vectors of $\B$ are removed in the step 2.
We explain the details of the steps 1 and 2 although the steps 3 and 4
are the same as the steps 2 and 3 of NCER.

Step 1 computes the singular value decomposition (SVD) of $\A$. 
The SVD of $\A$ takes the form
\begin{equation}
\A = \U\Sigma\V^\top.
 \label{Eq: SVD form}
\end{equation}
$\U$ is a $d$-by-$d$ orthogonal matrix 
consisting of left singular vectors $\u_1, \ldots, \u_d \in \Real^d$.
$\V$ is an $m$-by-$m$ orthogonal matrix 
consisting of right singular vectors $\v_1, \ldots, \v_m \in \Real^m$.
$\Sigma$ is a $d$-by-$m$ diagonal matrix 
$(\mbox{diag}(\sigma_1, \ldots, \sigma_d), \0)$
consisting of  singular values $\sigma_1, \ldots, \sigma_d$
with the relation $\sigma_1 \ge \cdots \ge \sigma_d$.
Here, the $\0$ is a $d$-by-$(m-d)$ zero matrix.
We pick up the $r$ largest singular values $\sigma_1, \ldots, \sigma_r$, 
and construct the $r$-by-$r$ diagonal matrix 
$\Sigma_r = \mbox{diag}(\sigma_1, \ldots, \sigma_r)$.
The matrix $\U_r\Sigma_r\V_r^\top$ is known to be the best rank-$r$ approximation matrix to $\A$ 
in measured by matrix $2$-norm.
Here, $\U_r = (\u_1, \ldots, \u_r) \in \Real^{d \times r}$ 
for the $r$ column vectors $\u_1, \ldots, \u_r$ of $\U$, 
and $\V_r = (\v_1, \ldots, \v_r) \in \Real^{m \times r}$ 
for those $\v_1, \ldots, \v_r$ of $\V$.
It should be noted that we have $\A = \U\Sigma\V^\top = \U_r\Sigma_r\V_r^\top$ 
if $\A$ is a separable matrix.
After the SVD computation of $\A$, step 1 constructs
\begin{equation*}
 \B = \Sigma_r \V_r^\top \in \Real^{r \times m}.
\end{equation*}
We call the matrix $\B$ a {\it dimensionally reduced matrix} for $\A$.

Step 2 constructs a diagonal matrix $\S$ that scales the column vectors
of $\B$ to lie on a hyperplane.  
There exists such a hyperplane since $\B$ has no zero column vectors.
We consider a hyperplane $\HC = \{ \x \in \Real^d : \w^\top \x = z\}$.
Here, $\w$ is a $d$-dimensional real vector such that the elements are all nonzero 
and $z$ is a nonzero real number.
The intersection points of column vectors $\b_i$ of $\B$ with $\HC$ are given as 
$s_i \b_i$ for $s_i = z  / \w^\top \b_i$. 
Thus, in this step, we construct $\S = \mbox{diag}(s_1, \ldots , s_m)$ by using the $s_i$.

We need to put an assumption on a basis number $r$ 
which is set as an input parameter in advance.
Step 3 computes an MVEE for the column vectors $\q_1, \ldots, \q_m$ 
of $\B\S \in \Real^{r \times m}$.
In the MVEE computation,
Assumption \ref{Assump: Full row rank} needs to be satisfied 
and the computation is allowed under the assumption.
Therefore, we put an assumption on a basis number $r$ 
that $r \le \mbox{rank}(\A)$ for a data matrix $\A$.
If the assumption holds, the $r$ largest singular values of $\A$ are positive, 
and thus, we have $\mbox{rank}(\B\S) = r$.

ER has been shown in \cite{Miz14} to have the following properties.
Let $\A$ be a separable matrix of the form (\ref{Eq: Separable matrix}).
Then, ER for an input data $(\A, r)$ with $r = \mbox{rank}(\A)$ 
returns the index set $\JC$ such that $\F = \A(\JC)$.
Let $\wt{\A}$ be a near-separable matrix of 
the form (\ref{Eq: Near-separable matrix}). If $\N$ satisfies 
$||\N||_2 \le \epsilon(\F, \K)$, then, 
ER for an input data $(\wt{\A}, r)$ with $r = \mbox{rank}(\A)$ 
returns the index set $\JC$
such that $||\F - \wt{\A}(\JC)|| \le \epsilon(\F,\K)$.
Here, $\epsilon(\F,\K)$ is a nonnegative real number determined by 
the basis matrix $\F$ and  weight matrix $\K$.
For the details, we refer the reader to Theorem 9 of the paper.

\begin{remark}
 The original description of ER in \cite{Miz14} does not contain the step 3 
 of the above algorithm description.
 If the input data matrix $\A$ is a separable one, we can assume without loss of generality
 that the column vectors of $\A$ lie on a hyperplane. 
 This is because we have $\A \D_1 = \F\D_2 \D_2^{-1} \W \D_1$ for nonsingular 
 matrices $\D_1$ and $\D_2$.
 Hence, we can skip the step 3 
 if $\A$ is a separable matrix and 
 the column vectors of $\A$ are scaled to lie on a hyperplane in advance.
\end{remark}

\subsection{Connection between NCER and ER Algorithms}
We observe the active point sets constructed 
in the step 2 of NCER and the step 3 of ER,
and then, see that the two sets coincide with each other 
under some assumptions on data points and parameters of the algorithms.
In the next steps, the algorithms select $r$ elements from the active point sets, 
and construct sets by collecting them.
The two sets are different in general.
But, we see that those sets coincide if we modify one step of ER.

We put an assumption on data points $\a_1, \ldots, \a_m \in \Real^d$ 
for NCER and ER.
\begin{assump} \label{Assump: Data points}
 Any of $m$ data points $\a_1, \ldots, \a_m \in \Real^d$ are nonnegative and 
 not a zero vector.
\end{assump}
The first part of Assumption \ref{Assump: Data points} is required for ER.
In below discussion, we need to handle a diagonal matrix
\begin{equation} \label{Eq: Degree mat}
 \D = \mbox{diag}(d_1, \ldots, d_m) \in \Real^{m \times m} \ \mbox{with} \
  d_i = \a_i^\top (\a_1 + \cdots + \a_m)
\end{equation}
for data points $\a_1, \ldots, \a_m$.
The second part of Assumption \ref{Assump: Data points} is used to ensure that the $\D$ is nonsingular.
Let $\A = (\a_1, \ldots, \a_m) \in \Real^{d \times m}$ for the data points $\a_1, \ldots, \a_m \in \Real^d$.
We put assumptions for the parameters of NCER and ER.
\begin{assump} \label{Assump: NCER}
 For the NCER algorithm,  we have the following settings.
 \begin{enumerate}[{\normalfont (a)}]
  \item A similarity function in the input is set as $k(\a_i,\a_j) = \a_i^\top\a_j$.
  \item A neighbor number in the input is set as $p = m$.
 \end{enumerate}
\end{assump}
\begin{assump} \label{Assump: ER}
 Let $\D$ be of (\ref{Eq: Degree mat}). 
 For the ER algorithm, we have the following settings.
 \begin{enumerate}[{\normalfont (a)}]
  \item A data matrix in the input is set as $\A\D^{-1/2}$.
  \item A basis number in the input is set as $r$ satisfying $r \le \mbox{rank}(\A)$.
  \item A diagonal matrix in the step 2 is set as $\S = \D^{-1/2}$.
 \end{enumerate}
\end{assump}
It should be noted that 
the $\D$ in Assumption \ref{Assump: ER} is nonsingular under Assumption \ref{Assump: Data points}.
Assumption \ref{Assump: NCER}(a) chooses a polynomial function of
(\ref{Eq: Poly}) with $b=0$ and $c=1$ as a similarity function.
This function gives a similarity of two data points as its inner product.
Assumption \ref{Assump: NCER}(b) means not to ignore small values of similarity,
and take into account all the values in a graph.
Assumption \ref{Assump: ER}(a) scales the data points $\a_1, \ldots, \a_m$ 
into $d_1^{-1/2} \a_1, \ldots, d_m^{-1/2} \a_m$ by the diagonal elements $d_1, \ldots, d_m$ 
of $\D$. We give a remark on this assumption in the last of this section.
We see below that the $\D$ corresponds to a degree matrix of a graph 
constructed by the similarity function $k$ and $p$-nearest neighbor set 
under Assumption \ref{Assump: NCER}.
As mentioned in the last part of Section~\ref{Subsec: ER algorithm},
Assumption \ref{Assump: ER}(b) ensures to allow us to carry out an MVEE computation 
in the step 3 of ER. 
We see below that Assumption \ref{Assump: ER}(c) makes the column vectors of 
dimensionally reduced matrix lie on a hyperplane.

\begin{theo} \label{Theo: Relation of NCER and ER}
 Assume that 
 Assumptions \ref{Assump: Data points}, \ref{Assump: NCER}, and \ref{Assump: ER}
 hold for $m$ input data $\a_1, \ldots, \a_m \in \Real^d$ and also the NCER and ER algorithms.
 Let $\IC_1$ be the index set of active points in the step 2 of NCER, and 
 $\IC_2$ be that in the step 3 of ER.  Then, we have $\IC_1 = \IC_2$.
\end{theo}

We prove this theorem in the remaining of this section.
First, we shall see in detail how vectors $\p_1, \ldots, \p_m \in \Real^r$ 
in the step 1 of NCER are constructed
under Assumptions \ref{Assump: Data points} and \ref{Assump: NCER}.
For the data points $\a_1, \ldots, \a_m \in \Real^r$, 
we let $\A = (\a_1, \ldots, \a_m) \in \Real^{d \times m}$. 
Under the assumptions, a graph Laplacian $\L$ for the data points is 
\begin{equation} \label{Eq: Graph Laplacian}
 \L = \D - \A^\top \A.
\end{equation}
Here, $\D$ is a degree matrix for an adjacency matrix $\A^\top\A$, and 
its form is given as (\ref{Eq: Degree mat}).
The normalized graph Laplacian $\widebar{\L}$ is of the form $\I - \D^{-1/2} \A^\top\A \D^{-1/2}$.
Hence, the eigenvalue decomposition of $\widebar{\L}$ can be obtained through the SVD of $\A\D^{-1/2}$. 
In a similar way to the description of (\ref{Eq: SVD form}),
we write it as 
\begin{equation} \label{Eq: SVD for scaled A}
 \A\D^{-1/2} = \U\Sigma\V^\top.
\end{equation}
Here, $\U$ and $\V$ are respectively
$d$-by-$d$ and $m$-by-$m$ orthogonal matrices, 
and the column vectors are respectively the left and right singular vectors of $\A\D^{-1/2}$.
$\Sigma$ is a $d$-by-$m$ diagonal matrix 
$(\mbox{diag}(\sigma_1, \ldots, \sigma_d), \0)$, 
and $\sigma_1, \ldots, \sigma_d$ are the singular values of $\A\D^{-1/2}$.
The eigenvalue decomposition of $\widebar{\L}$ can be written as
\begin{eqnarray} \label{Eq: Eigenvalue decomp}
 \widebar{\L} 
 &=& \D^{-1/2}\L \D^{-1/2}  \nonumber \\
 &=& \I - \D^{-1/2} \A^\top\A \D^{-1/2} \nonumber \\
 &=& \V(\I - \Sigma^\top \Sigma)\V^\top.
\end{eqnarray}
by using the expressions (\ref{Eq: Graph Laplacian}) and (\ref{Eq: SVD for scaled A}).
We arrange the eigenvalues of $\widebar{\L}$ in ascending order, 
and denote the $i$th eigenvalue by $\lambda_i$.
Namely, $\lambda_1 \le \cdots \le \lambda_m$ holds.
Then, from the expression (\ref{Eq: Eigenvalue decomp}), we have 
\begin{equation*}
 \lambda_i = \left\{
 \begin{array}{ll}
  1 - \sigma_i^2, & i = 1, \ldots, d, \\
  1               & i = d+1, \ldots, m
 \end{array}
\right.
\end{equation*}
for $1 = \sigma_1 \ge \cdots \ge \sigma_d \ge 0$.
Here, the relation $\sigma_1 = 1$ comes from the fact that 
a normalized graph Laplacian is 
positive semidefinite and has a zero eigenvalue.
The column vector $\v_i$ of $\V$ is the eigenvector of 
$\widebar{\L}$, and corresponds to the eigenvalue $\lambda_i$ and also singular value $\sigma_i$.
Let $\V_r = (\v_1, \ldots, \v_r) \in \Real^{m \times r}$.
Under the assumptions, we see that 
$\p_1, \ldots, \p_m$ in the step 1 of NCER are the column vectors of 
\begin{equation} \label{Eq: P}
 \P = \V_r^\top \D^{-1/2},
\end{equation}
and the column vectors $\v_1, \ldots, \v_r$ of $\V_r$ 
are the $r$ right singular vectors of $\A\D^{-1/2}$ 
in correspondence with
the $r$ largest singular values $\sigma_1, \ldots, \sigma_r$ 
with $1=\sigma_1 \ge \cdots \ge \sigma_r  \ge 0$.

Next, we shall see in detail how vectors $\q_1, \ldots, \q_m \in \Real^r$ 
in the step 2 of ER 
are constructed under Assumptions \ref{Assump: Data points} and \ref{Assump: ER}.
Under the assumptions, step 1 computes the SVD of $\A\D^{-1/2}$ and constructs
a dimensionally reduced matrix $\B$ for the matrix.
Since the $\D$ is equivalent to the degree matrix constructed 
in NCER under Assumptions \ref{Assump: Data points} and \ref{Assump: NCER},
the SVD of $\A\D^{-1/2}$ coincides with  that of (\ref{Eq: SVD for scaled A}).
Let $\V_r = (\v_1, \ldots, \v_r) \in \Real^{m \times r}$ 
for the $r$ right singular vectors $\v_1, \ldots, \v_r$ 
in correspondence with the $r$ largest singular values $\sigma_1, \ldots, \sigma_r$
such that $1=\sigma_1 \ge \cdots \ge \sigma_r > 0$.
The last strictly inequality $\sigma_r > 0$ comes from Assumption \ref{Assump: ER}(b).
Also, let $\Sigma_r = \mbox{diag}(\sigma_1, \ldots, \sigma_r) \in \Real^{r \times r}$ 
for the $r$ singular values $\sigma_1, \ldots, \sigma_r$. 
Then, the dimensionally reduced matrix $\B$ is given as $\Sigma_r\V_r^\top \in \Real^{r \times m}$ 
by using those $\Sigma_r$ and $\V_r$.
From Assumption \ref{Assump: ER}(b), Step 2 constructs 
\begin{equation}\label{Eq: Q}
\Q =  \Sigma_r\V_r^\top \D^{-1/2}  \in \Real^{r \times m}
\end{equation}
by using $\D$ of (\ref{Eq: Degree mat}).
The $\V_r$ and $\D$ coincide with those of (\ref{Eq: P}).
As already shown in (\ref{Eq: Eigenvalue decomp}),
the $\V_r$ corresponds the $r$ eigenvectors of 
normalized graph Laplacian for $\L$ of (\ref{Eq: Graph Laplacian}).
We see from Proposition  \ref{Prop: Vecotrs-lie-on-hyperplane} 
that $\Q$ of (\ref{Eq: Q}) does not have any zero column vectors 
and the column vectors $\q_1, \ldots, \q_m$ lie on a hyperplane 
$\{\x \in \Real^d : \e_1^\top\x = \tau\}$ for some nonzero real number $\tau$.

The column vectors $\q_1, \ldots, \q_m$ of $\Q$ in (\ref{Eq: Q})
can be thought of as the images of those $\p_1, \ldots, \p_m$ of $\P$ in (\ref{Eq: P})
under a linear transformation $\Sigma_r$.
{This transformation is nonsingular since 
$\sigma_1, \ldots, \sigma_r$ are all positive under Assumption \ref{Assump: ER}(b).
It is well known that an MVEE computation is invariant under a nonsingular linear transformation;
see \cite{Boy04} for instance.
Therefore, the following proposition holds.
\begin{lemm} \label{Lemm: Invariance of active point set}
 Let $\G \in \Real^{r \times r}$ be a nonsingular matrix.
 Consider  $\p_1, \ldots, \p_m, \q_1, \ldots, \q_m \in \Real^r$ 
 such that $\q_i = \G \p_i$  for $i = 1, \ldots, m$.
 Let $\IC_1$ be the set of active points 
 in the origin-centered MVEE for a set $\SC = \{\pm \p_1, \ldots, \pm \p_m\}$ 
 and $\IC_2$ be that for a set $\TC = \{\pm \q_1, \ldots, \pm \q_m\}$.
 Then, we have $\IC_1 = \IC_2$.
\end{lemm}
\begin{proof}
 The origin-centered MVEEs for each of sets $\SC$ and $\TC$ 
 are given by the optimal solutions of 
 problems $\MVEEProb(\SC)$ and $\MVEEProb(\TC)$.
 We denote by $\L^* \in \Real^{r \times r}$ and $\M^* \in \Real^{r \times r}$ 
 the optimal solutions of $\MVEEProb(\SC)$ and $\MVEEProb(\TC)$.
 Then, we have $\L^* = \G^\top\M^*\G$.
 Therefore, the active point set of 
 origin-centered MVEE for $\SC$ coincides with that for $\TC$.
\end{proof}

The proof of Theorem \ref{Theo: Relation of NCER and ER} is now obtained.
\begin{proof}[Proof of Theorem \ref{Theo: Relation of NCER and ER}.]
It follows from the above discussion and Lemma \ref{Lemm: Invariance of active point set}.
\end{proof}

The step 3 of NCER and also the step 4 of ER select $r$ elements 
from the active point sets, and construct the index sets of the elements.
Although the two index sets are different in general,
those coincide if we modify one step of ER.
In the step 1, we set $\B$ as $\V_r^\top$ instead of $\Sigma_r\V_r^\top$.
We denote the modified version of ER by {\it MER}.
We immediately have the following corollary.

\begin{coro} \label{Coro: Relation of NCER and MER}
 Assume that 
 Assumptions \ref{Assump: Data points}, \ref{Assump: NCER}, and \ref{Assump: ER} hold.
 Let $\JC_1$ be the set constructed in the step 3 of NCER, and 
 $\JC_2$ be that in the step 4 of MER.  Then, we have $\JC_1 = \JC_2$.
\end{coro}
Accordingly, 
if we perform MER, and classify $\q_1, \ldots, \q_m$ after the step 4 
by following the step 4 of NCER,
then, the obtained clusters coincide with those by NCER under the three assumptions.
We will give the demonstration in Section \ref{Subsec: Connection of NCER with MER and ER}.

\begin{remark}
 Assumption \ref{Assump: ER}(a) requires us to perform a data scaling in ER.
 The same data scaling can be found in \cite{Xu03}.
 Empirical results in the paper show
 that the scaling can enhance the performance of NMF based clustering.
\end{remark}

\section{Issues in Clustering by  K-means and NMF}
\label{Sec: Issues}

We recall the algorithms in K-means and NMF clustering, 
and then, mention the issues about the choice of initial points.
An empirical study will be reported to confirm the issues 
in Subsection \ref{Subsec: Performance evaluation}.
Given a data set $\SC$ and an integer number $r$.
Let $\a$ denote a data point in $\SC$.
In K-means, we set the function for the disjoint partitions $\SC_1, \ldots, \SC_r$ 
of $\SC$,
\begin{equation*}
 f(\SC_1, \ldots, \SC_r)  = \sum_{j=1}^r \sum_{\a \in \SC_j} ||\a - \c_j||_2^2
\end{equation*} 
where
\begin{equation} \label{Eq: Center}
 \c_j = \frac{\sum_{\a \in \SC_j}\a}{|\SC_j|}
\end{equation}
and consider the problem of minimizing the function $f$.
The $\c_j$ serves as a center of the partition $\SC_j$.
Although, conventionally, the number of partitions is denoted by $K$ in K-means,
we continue to use $r$ for consistency in this paper.
The minimization problem is hard to solve, and in fact, 
has been shown to be NP-hard in \cite{Dri04, Alo09}.
Thus, instead of the global optimal solution, 
we find the local solution by using an alternative procedure.
We minimize $f$ by alternatively fixing centers $\c_1, \ldots, \c_r$ and partitions $\SC_1, \ldots, \SC_r$.
When $\c_1, \ldots, \c_r$ are fixed, 
a data point $\a$ is assigned to $\SC_{j}$ having the nearest center $\c_{j}$
such that the index $j$ attains the minimum value of $||\a - \c_j||_2^2$ 
for $j = 1, \ldots, r$.
When $\SC_1, \ldots, \SC_r$ are fixed,
$\c_j$ is computed by following (\ref{Eq: Center}).
The K-means algorithm repeats the procedure until some stopping criteria are satisfied.
To start the alternative procedure,
it requires us to arbitrarily choose $\c_1, \ldots, \c_r$ and fix them in advance.
It is empirically known that the choice of initial centers is sensitive 
to the cluster construction;
some choices may make it possible to show good clustering performance, 
while others may not.
It is difficult to choose good initial points before running the algorithm
and the choice affects the cluster construction.
NC has the same issues since K-means is incorporated in it.

Given a $d$-by-$m$ nonnegative matrix $\A$ and an integer number $r$.
The column vectors $\a_i$ of $\A$ correspond to data points.
In NMF, we find a $d$-by-$r$ nonnegative matrix $\F$ 
and an $r$-by-$m$ nonnegative matrix $\W$ such that the product of $\F$ and $\W$ 
is as close to $\A$ as possible.
A natural way for the formulation is that we set the function for matrices 
$\F \in \Real^{d \times r}$ and $\W \in \Real^{r \times m}$
\begin{equation} \label{Eq: Func for NMF}
 f(\F, \W) = ||\F\W - \A||_F^2
\end{equation}
and consider the problem of minimizing the function $f$ under the nonnegativity 
constraints on $\F$ and $\W$.
Solving the problem is hard, and has been shown to be NP-hard in \cite{Vav09}.
However, the intractable problem can be reduced into 
tractable subproblems if either of $\F$ or $\W$ is fixed.
Each subproblem becomes a easily solvable NLS problem.
The NMF algorithm repeats to solve 
NLS problems by alternatively fixing $\F$ and $\W$ so as to minimize $f$
until some stopping criteria are satisfied.
This alternative procedure can be regarded as a BCD framework which 
is used for minimizing a nonlinear function; 
see \cite{Kim14} for the details of the framework.
After that, the algorithm classifies data points $\a_i$ 
to $r$ clusters by using the obtained $\F$ and $\W$.
The assignment is similar to the step 4 of NCER.
We regard the column vectors of $\F$ as the representative points of clusters, 
and the elements of column vectors of $\W$ as a contribution rate.
For a column vector $\w_i$,
we find the index $j$ of the largest element in $\w_i$, 
and assign a data point $\a_i$ to the $j$th cluster $\SC_j$.
The NMF algorithm employs a BCD framework,
and it requires us to arbitrarily choose $\F$ (or $\W$) and fix it in advance.
As is the case in K-means,
the choice of initial matrix is sensitive to the cluster construction.

There are many studies on the efficient implementation of NMF algorithm.
The main discussion points are in how to solve NLS problems.
Although various types of algorithms are proposed,
the multiplicative update algorithm of \cite{Lee99} may be the most popular one.
Empirical results in \cite{Lin07} and \cite{Kim08a} imply
that a projected gradient algorithm and an active set algorithm are faster 
and provide more accurate solutions than a multiplicative update.
We refer the reader to \cite{Kim14} for further discussion.

We finally mention the GNMF algorithm proposed in \cite{Cai11}. 
Empirical results in the paper imply that GNMF outperforms NC and NMF in clustering.
This is a variant of the NMF algorithm.
In GNMF, we add the regularization term $R$ to the function $f$ of (\ref{Eq: Func for NMF}), 
and set the function
\begin{equation} \label{Eq: Func for GNMF}
 f(\F, \W) =  ||\F\W - \A||_F^2 + \lambda R(\W)
\end{equation}
where 
\begin{eqnarray*}
 R(\W) 
 &=& \frac{1}{2} \sum_{j=1}^m \sum_{i=1}^m k_{ij} ||\w_i - \w_j||_2^2 \\
 &=& \mbox{tr}(\W^\top \L \W).
\end{eqnarray*}
The problem we consider is to minimize the function $f$ 
under the nonnegative constraints on $\F$ and $\W$.
Here, $k_{ij}$ are the elements of adjacency matrix $\K$ which is constructed 
for data points $\a_1, \ldots, \a_m$, and $\L$ is the corresponding graph Laplacian.
Also, $\lambda$ is a parameter and takes a nonnegative value.
The term $R$ is called a graph regularizer.
We interpret $\W$ as a dimensionally reduced data matrix for $\A$.
If data points $\a_i$ and $\a_j$ are similar to each other, 
the term forces the dimensionally reduced ones $\w_i$ and $\w_j$ to be similar.
The GNMF algorithm minimizes $f$ by using a BCD framework,
and solves subproblems by using an algorithm 
based on the notion of the multiplicative update.
After that, it applies K-means to the column vectors $\w_1, \ldots, \w_m$ of $\W$, 
and finds $r$ clusters.
Accordingly, there are two parts in GNMF that require us to set initial points in advance.

\section{Related Work on Connection of Spectral Clustering and NMF}
\label{Sec: Related work}

Some of previous studies discuss the relationship between spectral clustering and NMF.
In particular, the paper \cite{Din05} points out that a similarity can be found 
in the problem formulations in spectral clustering and NMF.
We again consider the normalized cut minimization problem $\OriProb$.
Since the matrix variable $\H$ are nonnegative, 
we construct the following relaxation problem by taking account of it.
\begin{equation*}
  \begin{array}{cccc}
   \mbox{minimize} & \mbox{tr}(\H^\top \L \H)  & \mbox{subject to}  
    &  \H^\top\D\H = \I \ \mbox{and} \ \H \ge \0.
  \end{array}
\end{equation*}
Under the change of variable $\G = \D^{1/2}\H$,
this is equivalent to 
\begin{equation*}
  \begin{array}{cccc}
   \mbox{minimize} & \mbox{tr}(\G^\top \D^{-1/2}  \L \D^{-1/2} \G)  & \mbox{subject to}  
    &  \G^\top \G = \I \ \mbox{and} \ \G \ge \0
  \end{array}
\end{equation*}
since $\D$ is a diagonal matrix such that the diagonal elements are all positive.
Furthermore, 
the object function is $\mbox{tr}(\I) - \mbox{tr}(\G^\top \widebar{\W} \G)$ and we have 
$\mbox{tr}(\I) -2\mbox{tr}(\G^\top \widebar{\W} \G) + \mbox{tr}(\widebar{\W}^\top \widebar{\W}) 
= ||\G\G^\top - \widebar{\W}||_F^2$.
We here let $\widebar{\W} = \D^{-1/2} \W \D^{-1/2}$.
Thus, the above relaxation problem is essentially equivalent to
\begin{equation*}
 \begin{array}{cccc}
  \mbox{minimize} & ||\G \G^\top  - \widebar{\W}||_F^2  & \mbox{subject to} 
   &  \G^\top \G = \I \ \mbox{and} \ \G \ge \0.
 \end{array}
\end{equation*}
Hence, if we drop the orthogonal constraint $\G^\top \G = \I$ from there,
then, it can be thought of as the special case of NMF 
since $\widebar{\W}$ is a nonnegative matrix.
However, to the best of the author's knowledge,
there may be no work to rigorously discuss the similarity of algorithms 
for spectral clustering and NMF.

\section{Experiments}
\label{Sec: Experiments}

We will show an empirical study 
to investigate the performance of NCER.
The first experiments visualize the products obtained from NCER 
for a small data set.
In the experiments, we will see the relationship 
between the performance of NCER and Observation~\ref{Obs: Dist of eigenvec}.
The second experiments compare the performance of NCER with existing algorithms.
The third experiments display how the performance of NCER varies with the neighbor number $p$.
All experiments were conducted on a 3.2 GHz CPU processor and 12 GB memory.

\subsection{Algorithm Implementation, Data Sets, and Measurements}
\label{Subsec: Implementation, data sets, and measurements}

All algorithms used in the experiments were implemented on MATLAB.
We describe the implementation details as follows.

\begin{itemize}
 \item {\bf NCER.}
       Three computation require to be carried out:
       eigenvalue and eigenvector computation, 
       MVEE computation, and NLS problem solving.
       We used MATLAB commands {\tt eigs} and {\tt lsqnonneg} 
       for the first and third computation.
       The {\tt lsqnonneg} employs the active set algorithm for an NLS problem.
       For the second computation,
       we performed the interior-point algorithm in a cutting plane framework.
       It has been shown empirically in \cite{Sun04, Ahi08}
       that the use of cutting plane accelerates the efficiency of 
       interior-point algorithm for an MVEE problem,
       and makes it possible to handle large problems.
       We used the software package SDPT3 \cite{Toh99b}
       to perform the interior-point algorithm.

 \item {\bf NC.}
       In addition to eigenvalue and eigenvector computation, 
       K-means requires to be performed. 
       We used a MATLAB command {\tt kmeans} for performing it and the {\tt eigs} command.

 \item {\bf NMF.}
       NLS problems require to be solved 
       in the BCD framework for minimizing $f$ of (\ref{Eq: Func for NMF}).
       As mentioned in Section~\ref{Sec: Issues}, 
       there are various possibilities for the choice of algorithms for solving NLS problems.
       We used the code available from the author's website of \cite{Lin07}.
       The code employs the projected gradient algorithm for the problems, 
       and it tends to show better clustering performance than others.

 \item {\bf GNMF.}
       In addition to K-means, 
       the BCD framework for minimizing $f$ of (\ref{Eq: Func for GNMF}) requires to be performed.
       We used the code available from the first author's website of \cite{Cai11}
       for performing it and the {\tt kmeans} command. 

\end{itemize}

We used image data sets and document data in the experiments.
We give an explanation for the data sets as follows.

\begin{itemize}
 \item {\bf COIL20.}
       This is a data set of $20$ object images.
       The images are taken by turning the objects with $360$ degree rotation 
       in each $5$ degree interval.
       The data set consists of $72$ images per object, and
       contains a total of $1440$ images. 
       The size of images is $128$-by-$128$ pixels with $256$ grayscale intensities.
       The data set is available from the website 
       (\url{http://www.cs.columbia.edu/CAVE/software/softlib/coil-20}).
       Although two types of data sets, processed and unprocessed ones, are provided,
       we used the processed version.

 \item {\bf JAFFE.}
       This is a data set of facial images of $10$ Japanese female models.
       It consists of $7$ different types of facial expressions per model, and 
       contains a total of $213$  images.
       The size of images is $256$-by-$256$ pixels with $256$ grayscale intensities.
       The data set is available from the website 
       (\url{http://www.kasrl.org/jaffe_info.html}).

 \item {\bf ORL.}
       This is a data set of facial images of $40$ human models.
       The images are taken by changing facial expressions and lightning.
       The data set consists of $10$ different types of facial images per model, 
       and contains a total of $400$ images.
       The size of images is $112$-by-$92$ pixels with $256$ grayscale intensities.
       The data set is available from the website 
       (\url{http://www.cl.cam.ac.uk/research/dtg/attarchive/facedatabase.html}).

 \item {\bf MNIST.}
       This is a data set of handwritten digits from $0$ to $9$.
       The data set is constructed by using some part of data sets 
       available from National Institute of Standards and Technology.
       The data set contains a total of $10000$ images.
       The size of images is $28$-by-$28$ pixels with $256$ grayscale intensities.
       The data set is available from the website 
       (\url{http://yann.lecun.com/exdb/mnist}).
       Although two types of data sets, 
       training data set and test data set, are provided,
       we used the test data set.

 \item {\bf USPS.}
       Along with MNIST, this is also a data set of handwritten digits from $0$ to $9$.
       The data set contains a total of $9298$ images.
       The size of images is $16$-by-$16$ pixels,
       and the grayscale intensities are scaled into the interval from $-1$ to $1$.
       The data set is available from the website 
       (\url{http://www.gaussianprocess.org/gpml/data}).

 \item {\bf ReutersTOP10.}
       This data set was constructed by using some part of Reuters-21578 corpus.
       The corpus contains 21578 news articles appeared on the Reuters newswire in 1987,
       and the articles are manually classified into 135 topic groups.
       The corpus size can be reduced into 8293 articles in 65 topic groups
       by discarding the articles belonging to multiple topics.
       We picked up the top 10 largest topic groups from the size-reduced corpus, 
       and constructed a data set by collecting all articles in the topic groups.
       The data set contained 7285 articles with 18933 words.
       We denote it by ReutersTOP10.
       The Reuters-21578 corpus is available from 
       the UCI Knowledge Discovery in Databases Archive
       (\url{http://kdd.ics.uci.edu}).
       The size-reduced corpus is available from the website 
       (\url{http://www.cad.zju.edu.cn/home/dengcai}).

 \item {\bf BBC.}
       This is a corpus of news articles appeared on 
       the BBC news website in 2004-2005.
       The news articles are chosen from five topic groups, and 
       are preprocessed by applying 
       stemming, stop-word removal, and low word frequency filtering.
       The data set contains 2225 articles with 9636 words, and 
       is available from the website 
       (\url{http://mlg.ucd.ie/datasets/bbc.html}).

\end{itemize}

We generated data matrices by using the above data sets.
Consider the case of a image data set 
such that it consists of $m$ grayscale images of $s$-by-$t$ pixels.
In this case, we vectorized each image data into an $(s \times t)$-dimensional vector, and 
constructed an $(s \times t)$-by-$m$ matrix by stacking the vectors on the columns.
The images in all the data sets except USPS have 256 grayscale intensities.
Thus, the element values of the data matrices ranged from $0$ to $255$.
For USPS data set, we shifted the element values of the data matrix by $1$ 
so as to range from $0$ to $2$.
Consider the case of a document data set 
such that it consists of $m$ document with $d$ words.
We constructed a $d$-by-$m$ matrix. 
The elements of the matrix represent the frequency of words appeared in a document, 
and the appearance frequency of words was evaluated by tf-idf weighting scheme.
For the details of the scheme, 
we refer the reader to Section 6.2 of \cite{Man08}.

On each data set,
we manually classified the data into groups under predefined criteria 
such that those become the disjoint partitions of the data set.
The groups manually constructed are referred to as {\it classes} in contrast to 
clusters returned by an algorithm.
In case of image data sets, 
the data were classified according to object, human model, or digit.
In case of document data sets, 
the data were classified according to topic group.
Table \ref{Tab: Data matrices} summarizes 
the type of data, size of data matrix, and number of classes in the data sets.

\begin{table}[h]
 \centering
 \caption{Data type, data matrix size, and the number of classes in the data sets.} 
 \label{Tab: Data matrices}
 \begin{tabular}{lccrrcr}
 \hline
   & Data type  & & \multicolumn{2}{c}{Matrix size}  & & $\#$ Classes  \\
   &            & & $d$ & $m$ & & \\
  \cline{2-2} \cline{4-5} \cline{7-7}
 COIL20    & Object image & & 16384 & 1440  & & 20 \\
  JAFFE    & Facial image & & 65536 & 213   & & 10 \\
  ORL      & Facial image & & 10304 & 400   & & 40 \\
  MNIST    & Digit image  & & 784   & 10000 & & 10 \\
  USPS     & Digit image  & & 256   & 9298  & & 10 \\
  ReutersTOP10  & News article & &  18933 & 7285 & & 10 \\
  BBC      & News article & & 9635 & 2225   & & 5 \\
 \hline
 \end{tabular}
\end{table}

Two measurements were used 
for the evaluation of clusters constructed by an algorithm.
One is accuracy (AC) and another is normalized mutual information (NMI).
Let $\Omega_1, \ldots, \Omega_r$ be classes for a data set, 
and $\CC_1, \ldots, \CC_r$ be clusters returned by an algorithm for the data set.
In AC, we compute the correspondence relationship between 
$\Omega_1, \ldots, \Omega_r$ and $\CC_1, \ldots, \CC_r$ 
to maximize the total number of common elements $|\Omega_i \cap \CC_j|$.
Such a correspondence can be obtained 
by solving an assignment problem.
The indices of the classes and clusters are reattached to 
follow the obtained correspondence order.
Then, AC is defined as
\begin{equation*}
\frac{1}{m}(|\Omega_1 \cap \CC_1| + \cdots + |\Omega_r \cap \CC_r|).
\end{equation*}
In NMI, we compute
the mutual information $I(\Omega, \CC)$ for $\Omega$ and $\CC$, 
and the entropies $H(\Omega)$ and $H(\CC)$ for each of $\Omega$ and $\CC$.
Here, $\Omega$ and $\CC$ denote $\{\Omega_1, \ldots, \Omega_r\}$ and $\{\CC_1, \ldots, \CC_r\}$.
Then, NMI is defined as
\begin{equation*}
 \frac{I(\Omega, \CC)}{\frac{1}{2}(H(\Omega) + H(\CC))}.
\end{equation*}
We refer the reader to Section 16.3 of \cite{Man08} 
for the precise definitions of mutual information and entropy.
Both measurements take the values ranging from 0 to 1.
If clusters and classes are similar to each other,
the values are close to one; otherwise, those are close to zero.

\subsection{Illustration to See the Relationship between 
Clustering Performance and Observation \ref{Obs: Dist of eigenvec}}
\label{Subsec: Illustration}

The clustering performance of NCER can be expected to be high 
if Observation~\ref{Obs: Dist of eigenvec} holds.
Experiments were conducted with the purpose of illustrating this.
We visualized the products produced by NCER on a small data set
to see whether Observation~\ref{Obs: Dist of eigenvec} holds or not.
The data set was constructed by picking image data
in three classes of MNIST corresponding to handwritten digits 4, 5, and 6.
The data set contained 2832 images of $28$-by-$28$ pixels.
We conducted NCER by using four different neighbor numbers $p$,
namely, $5, 944, 1888,$ and $2832$, chosen 
so as to divide the range from 0 to 2832 into three almost equal parts.
The other input parameters were set such that 
$r$ is $3$, and $k$ is the inner product of two data points.

Table \ref{Tab: Exp1} and Figure \ref{Fig: Exp1} display the experimental results.
The table summarizes the ACs and NMIs for the four neighbor numbers.
The figures may need some explanation.
The size of the data matrix is $784$-by-$2832$.
The vectors $\p_1, \ldots, \p_{2832}$ constructed in the step 1 of NCER
are $3$-dimensional vectors
due to $r=3$ and lie on a hyperplane $\HC = \{\x \in \Real^3 : \e_1^\top \x  = \tau\}$.
The MVEE for the set $\SC = \{\pm \p_1, \ldots, \pm \p_{2832}\}$ is $3$-dimensional and 
is centrally symmetric, having the origin as the center.
The figures show the intersections of 
$\p_1, \ldots, \p_{2832}$ and MVEE with $\HC$.
The figures display four cases for each neighbor number.
The colored points correspond to $\p_1, \ldots, \p_{2832}$, and 
red, blue and green respectively represent the three classes.
The ellipsoids surrounded by the black lines are the MVEEs, and 
the squares surrounds  active points.

\begin{figure}[h]
 \centering
\begin{minipage}{.48\linewidth}
 \includegraphics[width=\linewidth]{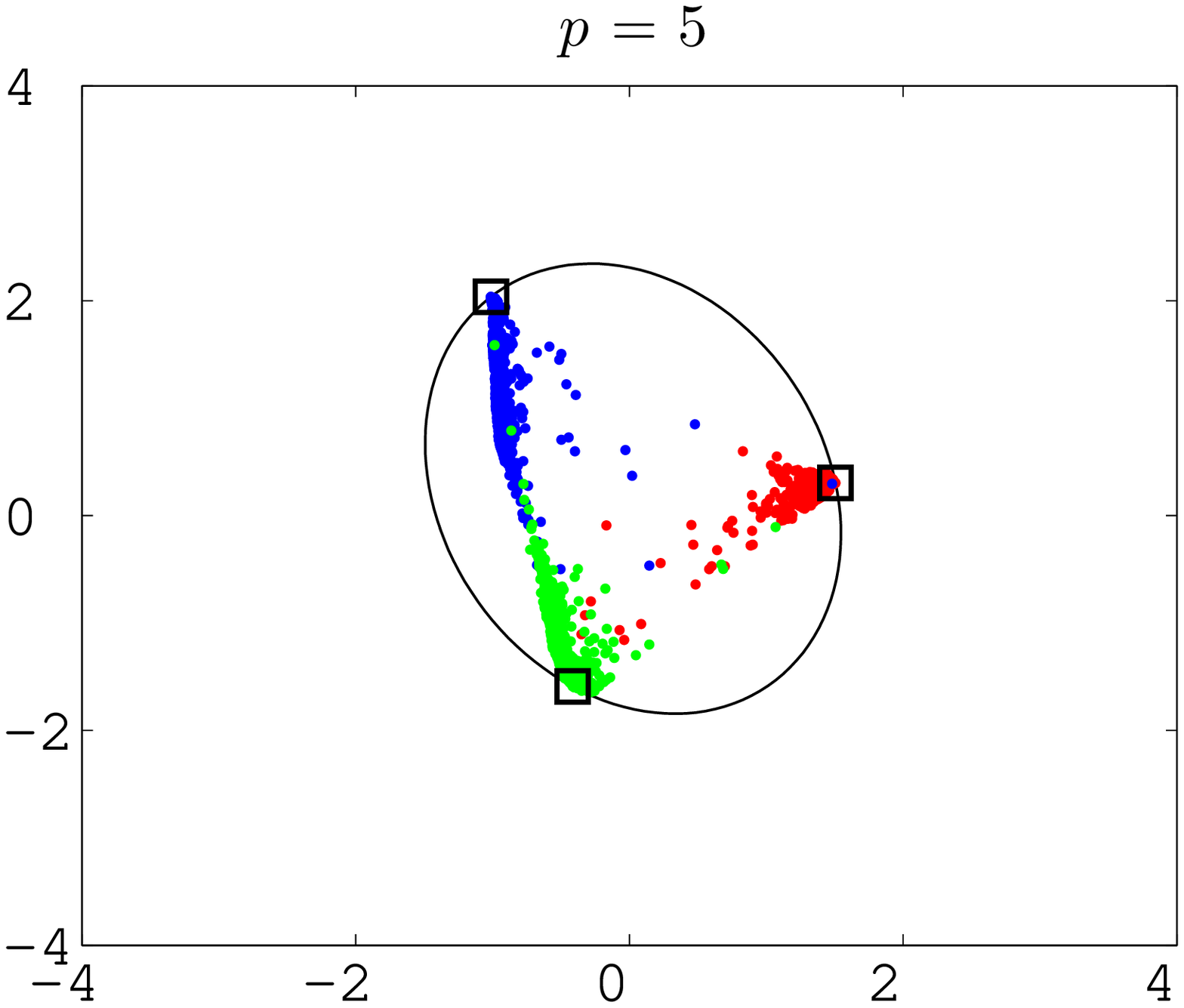}
\end{minipage}
\begin{minipage}{.48\linewidth}
 \includegraphics[width=\linewidth]{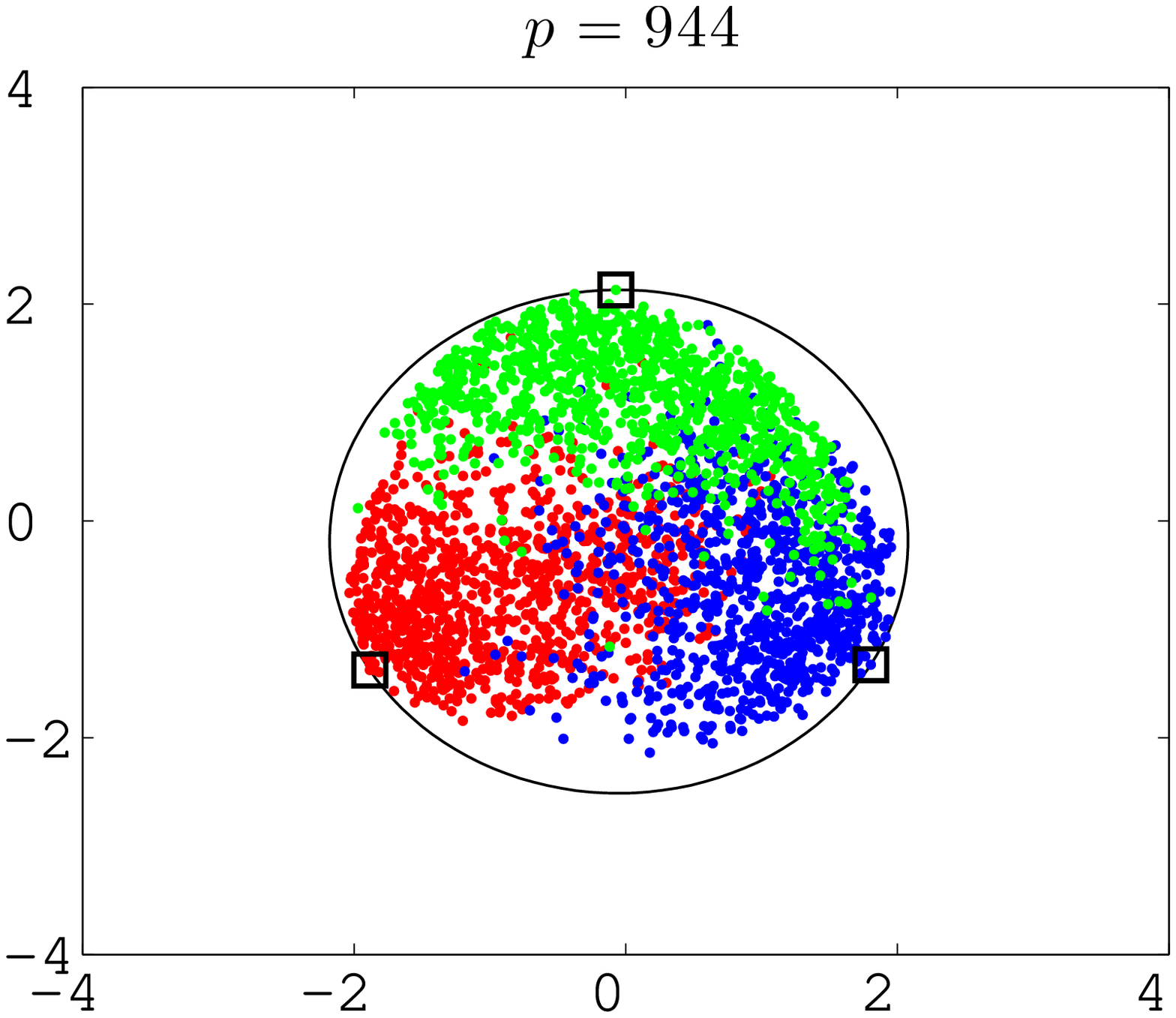}
\end{minipage}
\begin{minipage}{.48\linewidth}
 \includegraphics[width=\linewidth]{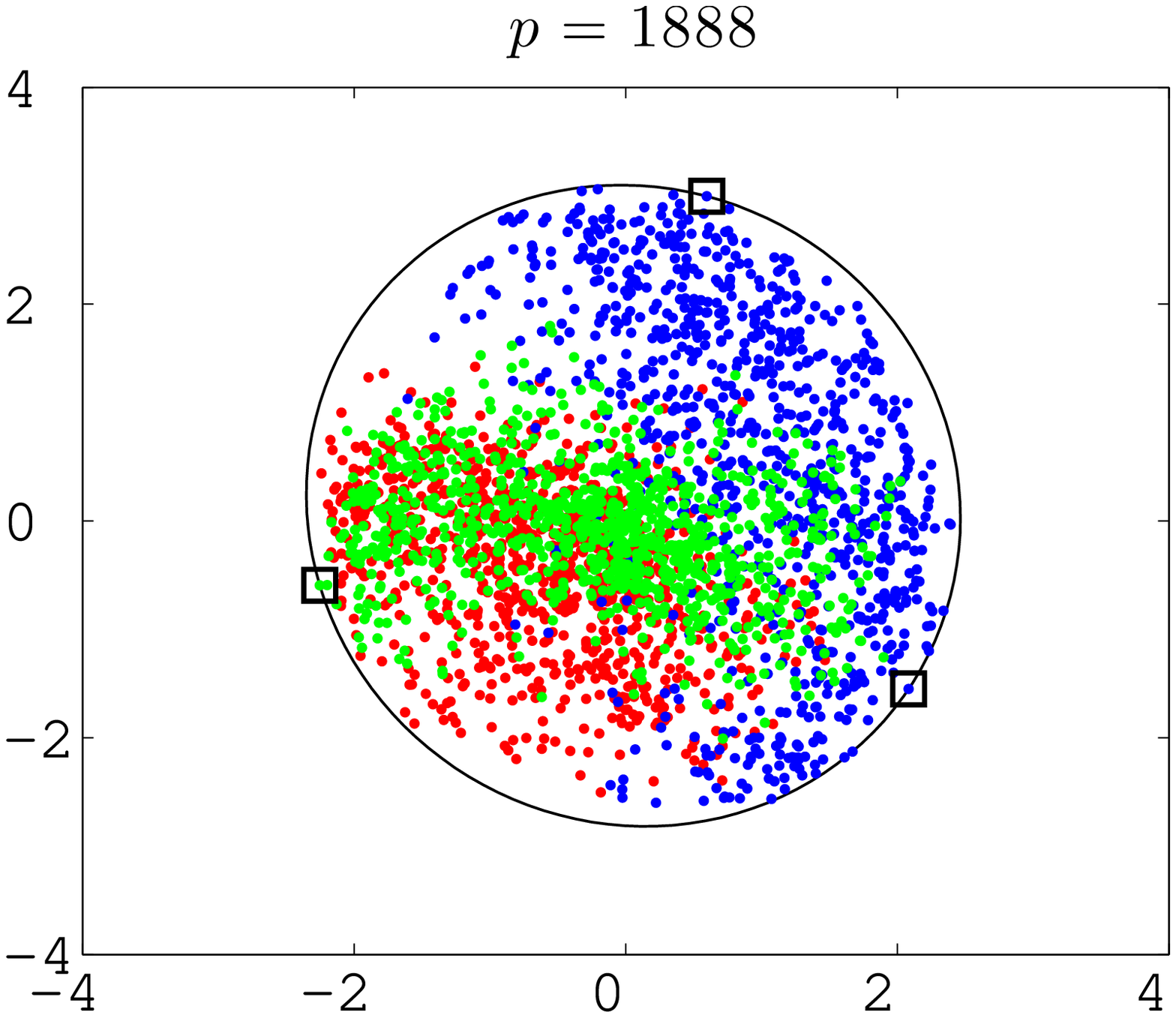}
\end{minipage}
\begin{minipage}{.48\linewidth}
 \includegraphics[width=\linewidth]{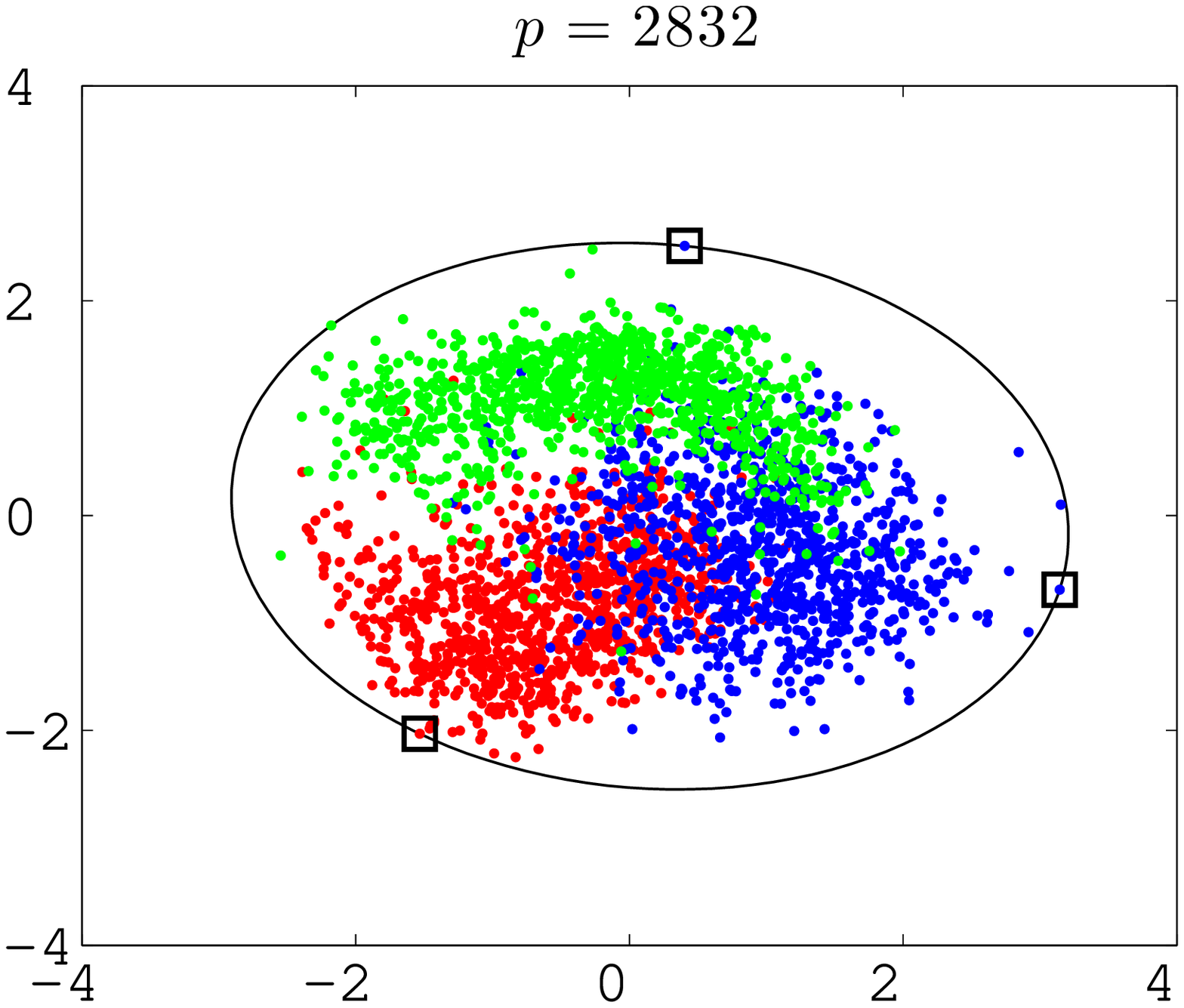}
\end{minipage}
\caption{Visualization of 
 $\p_1, \ldots, \p_{2832}$ and MVEEs on a hyperplane $\HC$.
 These were produced using NCER for a data set 
 consisting of three classes in MNIST
 and choosing four different neighbor numbers $p$.
 The red, blue, and green points correspond to 
 the three classes for the handwritten digits 4, 5, and 6.
 The ellipsoids surrounded by the black lines are 
 the intersections of MVEEs for the set $\SC = \{\pm \p_1, \ldots, \pm \p_{2832}\}$ 
 with $\HC$.
 The points in the squares are active points.
 The case of $p=5$ is at the top-left corner,
 that of $p=944$ at the top-right corner,
 $p=1888$ at the bottom-left corner, and 
 $p=2832$ at the bottom-right corner.}
 \label{Fig: Exp1}
\end{figure}

\begin{table}[h]
 \centering
 \caption{Cluster evaluation obtained by NCER
 for four different neighbor numbers $p$.}
 \label{Tab: Exp1}
\begin{tabular}{lcccc}
 \hline
     & $p=5$ & $p=944$ & $p=1888$ & $p=2832$ \\
 \cline{2-5}
 AC   & 0.987 & 0.829 & 0.546 & 0.799 \\
 NMI  & 0.934 & 0.496 & 0.258 & 0.460 \\
 \hline
\end{tabular}
\end{table}

We can say from the figures
that Observation~\ref{Obs: Dist of eigenvec} holds in the case of $p=5$.
We see from the table that the AC and NMI of this case are high.
Meanwhile, it is difficult to say that Observation~\ref{Obs: Dist of eigenvec} holds 
in the other cases.
In particular, the figure with $p=1888$
displays that three types of data points are mixed and not separated.
In fact, the AC and NMI of this case are low.
We see from the figures with $p=944$ and $p=2832$
that the data points group around three active points 
although some of the points are mixed.
Hence, the ACs and NMIs of these cases are higher than those of $p=1888$.

\subsection{Performance Evaluation} 
\label{Subsec: Performance evaluation}

As mentioned in Section~\ref{Sec: Issues},
the existing algorithms, NC, NMF, and GNMF, require
the initial points to be chosen before they are run, and the choice affects the clustering performance.
Hence, their performance may deteriorate 
with a choice of initial points which bring a low performance.
Meanwhile, NCER does not have such issues. 
Hence, we can expect that its performance will be stable and as high as that of NC.

Experiments were conducted in the purpose of comparing the performance of NCER 
with the existing algorithms.
In the experiments of the existing algorithms, 
we chose multiple initial points and measured the worst, best, and average performance.
We randomly generated 100 initial points for running the K-means algorithm in NC and GNMF
and 100 initial points for running the BCD framework
to minimize the functions of (\ref{Eq: Func for NMF}) and (\ref{Eq: Func for GNMF}) in NMF and GNMF. 
By using these initial points, 
we performed a total of 100 trials for each of NC and NMF 
and a total of 10000 trials for GNMF on each data set.
Meanwhile, we carried out one trial for NCER on each data set.
The parameters of NCER and NC were set such that 
$p$ is 5, $r$ is the number of classes in the data set, and 
$k$ is the inner product of two data points.
We used the COIL20, JAFFE, ORL, MNIST, and USPS image data sets.

\begin{table}[h]
 \centering
 \caption{Performance evaluation of the algorithms in AC.}
 \label{Tab: Ex2_AC}
\begin{tabular}{lcccccccccccccc}
 \hline
       & NCER &
       & \multicolumn{3}{c}{NC} & & \multicolumn{3}{c}{NMF} & & \multicolumn{3}{c}{GNMF} \\
       &      & & ave & min & max & & ave  & min  & max & & ave & min & max \\
 \cline{2-2} \cline{4-6} \cline{8-10} \cline{12-14}
COIL20 & \textbf{0.831} & & 0.585 & 0.386 & 0.753 & & 0.574 & 0.449 & 0.680 & & 0.681 & 0.465 & 0.817 \\
JAFFE  & \textbf{0.948} &  & 0.753 & 0.469 & 0.981 &  & 0.203 & 0.174 & 0.230 &  & 0.789 & 0.531 & 0.981 \\
ORL    & \textbf{0.792} &  & 0.667 & 0.570 & 0.745 &  & 0.595 & 0.515 & 0.662 &  & 0.645 & 0.555 & 0.700 \\
MNIST  & 0.663 &  & \textbf{0.664} & 0.500 & 0.810 &  & 0.522 & 0.429 & 0.532 &  & 0.632 & 0.449 & 0.727 \\
USPS   & \textbf{0.802} &  & 0.710 & 0.395 & 0.937 &  & 0.683 & 0.563 & 0.756  & & 0.750 & 0.289 & 0.940 \\
 \hline
\end{tabular}

 \caption{Performance evaluation of the algorithms in NMI.}
 \label{Tab: Ex2_NMI}
\begin{tabular}{lcccccccccccccc}
 \hline
       & NCER & 
       & \multicolumn{3}{c}{NC} & & \multicolumn{3}{c}{NMF} & & \multicolumn{3}{c}{GNMF} \\
       &       & & ave & min & max & & ave  & min  & max & & ave & min & max  \\
 \cline{2-2} \cline{4-6} \cline{8-10} \cline{12-14}
COIL20 & \textbf{0.933} &  & 0.813 & 0.695 & 0.887 & & 0.712 & 0.666 & 0.748  & & 0.857 & 0.774 & 0.919 \\
JAFFE  & \textbf{0.938} &  & 0.874 & 0.721 & 0.974  & & 0.109 & 0.068 & 0.141 & & 0.910 & 0.794 & 0.974 \\
ORL    & \textbf{0.875} &  & 0.837 & 0.787 & 0.869 &  & 0.788 & 0.759 & 0.811  & & 0.825 & 0.797 & 0.842 \\
MNIST  & \textbf{0.755} &  & 0.735 & 0.642 & 0.779  & & 0.460 & 0.426 & 0.467  & & 0.719 & 0.646 & 0.745  \\
USPS   & \textbf{0.834} &  & 0.800 & 0.622 & 0.877  & & 0.620 & 0.566 & 0.641  & & 0.813 & 0.574 & 0.881  \\
 \hline
 \end{tabular}

\end{table}

Tables \ref{Tab: Ex2_AC} and \ref{Tab: Ex2_NMI} summarize 
the ACs and NMIs of the algorithms on each data set.
Since multiple trials for NC, NMF, and GNMF were conducted,
the statistics of the measurements are shown.
The columns labeled ``ave'', ``min'', and ``max'' list the average, minimum, and maximum values
of the corresponding measurements.
For each data set,
we compared the AC and NMI of NCER with the average ACs and NMIs of NC, NMF, and GNMF, 
and the tables show the highest values in boldface.
We can see that 
the ACs and NMIs of NCER are higher than 
the average ACs and NMIs of the existing algorithms, except MNIST.
NC and GNMF outperform NMF.
For JAFFE and USPS,
the maximum ACs and NMIs of NC and GNMF are higher than the ACs and NMIs of NCER,
but their minimum ACs and NMIs are considerably lower than their maxima.
Hence, the averages get worse.
For COIL20 and ORL,
we can also see that 
NC and GNMF have gaps between the minimum and maximum values of ACs and NMIs.
For MNIST, although the AC of NCER is lower than the average AC of NC,
the difference is quite small.
Furthermore, the minimum AC of NC is lower than the AC of NCER.
This indicates that there exists an initial point such that NC is inferior to NCER in AC.
Consequently, we see that 
the performances of NC, NMF, and GNMF depend on the choice of the initial points and 
the average performance tends to get worse because some initial points result in poor performance.
Meanwhile, NCER is a stable clustering algorithm and has high performance.

Let us mention the computational times of algorithms.
The experiments showed that although the computational time of NCER is longer than that of NC,
the difference is not so large.
The biggest difference was in the case of USPS;
NCER spent 81.4 seconds, while NC spent 67.0 seconds on average.
It should be noted that the experimental results showed 
that the bottleneck in computational time is in solving the eigenvalue problem.

\subsection{Connection of NCER with MER and ER}
\label{Subsec: Connection of NCER with MER and ER}

Finally, experiments were conducted to see
how the performance of NCER varies with the neighbor number $p$.
We used the COIL20 and MNIST image data sets and 
the ReutersTOP10 and BBC document data sets.

We set the input parameters for NCER as follows.
$p$ were chosen so as to increase by some unit size $s$ 
such that $p \in \{5, s, 2s, \ldots, cs, m\}$.
Here, $s$ and $c$ are integers
such that $s$ is strictly greater than $5$ and $c$ satisfies $cs < m \le (c+1)s$.
The unit sizes $s$ of each data set were set as follows: 
$s=30$ on COIL20, $s=200$ on MNIST, $s=200$ on ReutersTOP10, and $s=50$ on BBC.
The other input parameters were set such that 
$r$ is the number of classes in the data set,
and $k$ is the inner product of two data points.
We set the parameters for MER and ER as follows.
Let $\D$ be the matrix of (\ref{Eq: Degree mat}).
The data matrix $\A$ was scaled to $\A\D^{-1/2}$, 
and the scaled data matrix was used as the input for the algorithms.
The matrix $\S$ in step 2 of the algorithms was set as $\D^{-1/2}$.
The input parameter $r$ was chosen as the number of classes in a data set.

None of the data sets contained any data that corresponded to a zero vector, 
and the data matrices $\A$ satisfied $r \le \mbox{rank}(\A)$.
Hence,  Assumptions  \ref{Assump: Data points}, \ref{Assump: NCER}, and \ref{Assump: ER} 
held when the neighbor number $p$ for NCER was set as $m$.
Accordingly, Theorem \ref{Theo: Relation of NCER and ER} 
and Corollary \ref{Coro: Relation of NCER and MER} ensured that, when $p=m$, 
the clusters returned by NCER would not be far from those returned by ER 
and would coincide with those returned by MER.

Figure \ref{Fig: Ex3} depicts the graphs 
for showing the ACs and NMIs of NCER, MER, and ER versus neighbor number $p$.
The red points connected by the red line plot the ACs and NMIs of NCER.
The blue and green lines are the ACs and NMIs of MER and ER.
Table \ref{Tab: Exp3} summarizes the ACs and NMIs of NCER with $p=m$, MER, and ER.
We can see from the figure and table that the ACs and NMIs of NCER coincide with those of MER
when $p=m$.

The graphs on the COIL20 and MNIST image data sets indicate
that the clustering performance of NCER increases as the neighbor number $p$ gets close to zero.
However, the graphs on the ReutersTOP10 and BBC document data sets indicate
that the performance of NCER deteriorates when $p$ is a small number close to zero.
This difference may have come from the differences 
in the degree of similarity in each class of the data sets.
The data in the same class are quite similar to each other 
in the image data sets.
For instance, regarding COIL20,
the image data in a class were taken by turning an object through some interval of degrees.
Thus, the image data in the same class may retain a similarity structure 
even as $p$ gets close to zero.
However, this may not be the case in the document data sets, 
since the data in the same class are not as similar to each other as those of the image data sets.

\begin{figure}[p]
 \centering

\begin{minipage}{.4\linewidth}
 \includegraphics[width=\linewidth]{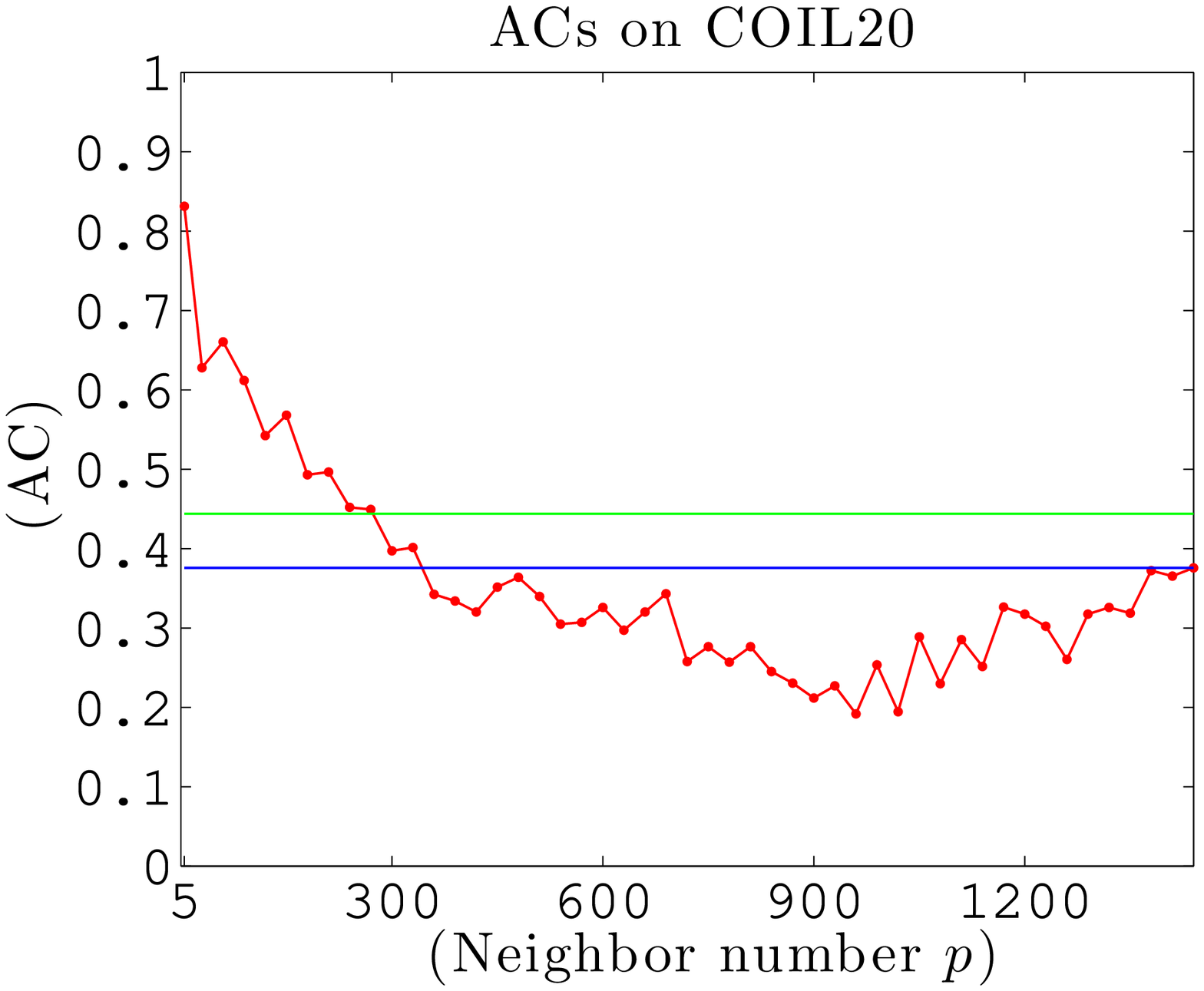}
\end{minipage}
\begin{minipage}{.4\linewidth}
 \includegraphics[width=\linewidth]{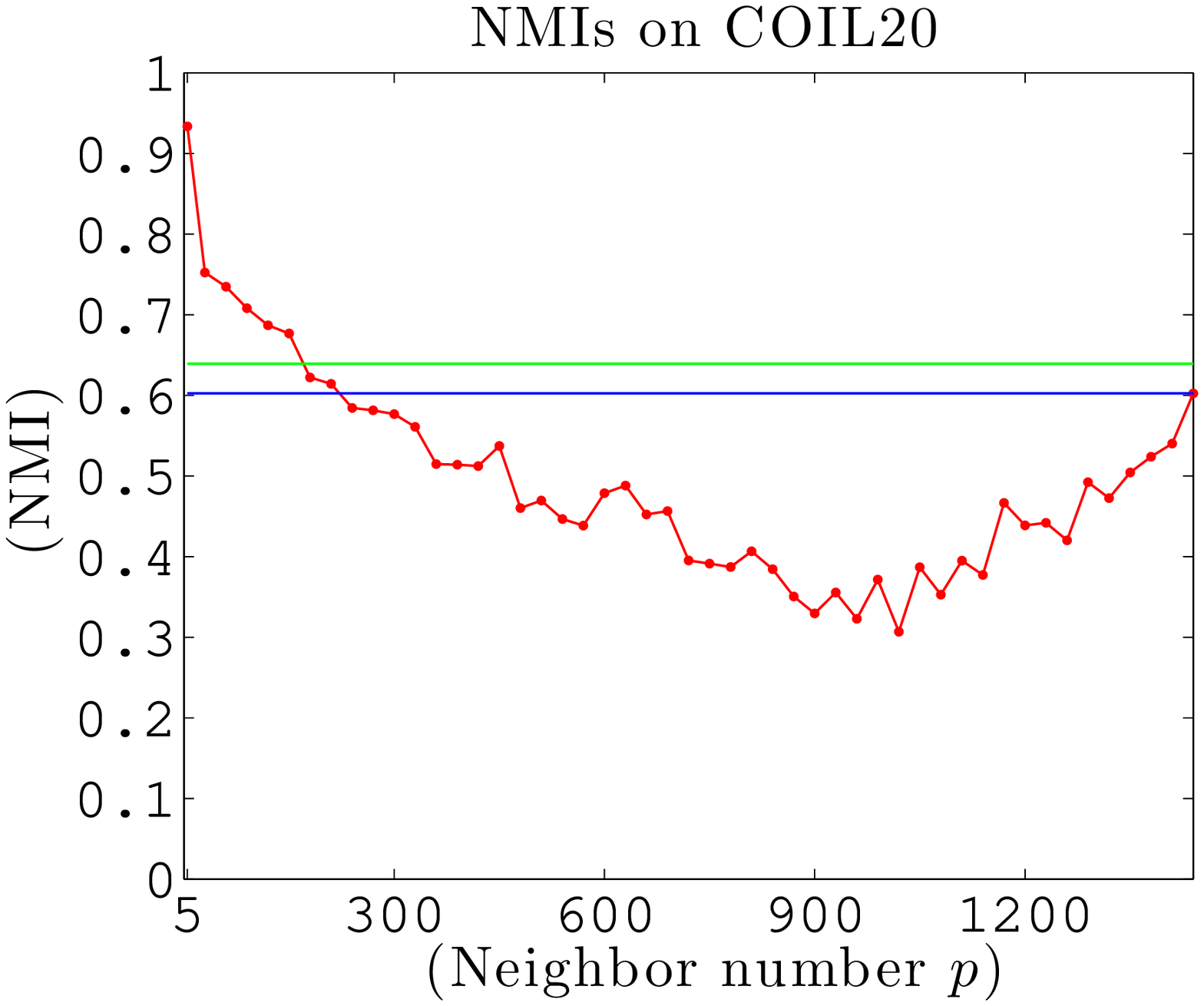}
\end{minipage}
\smallskip

\begin{minipage}{.4\linewidth}
 \includegraphics[width=\linewidth]{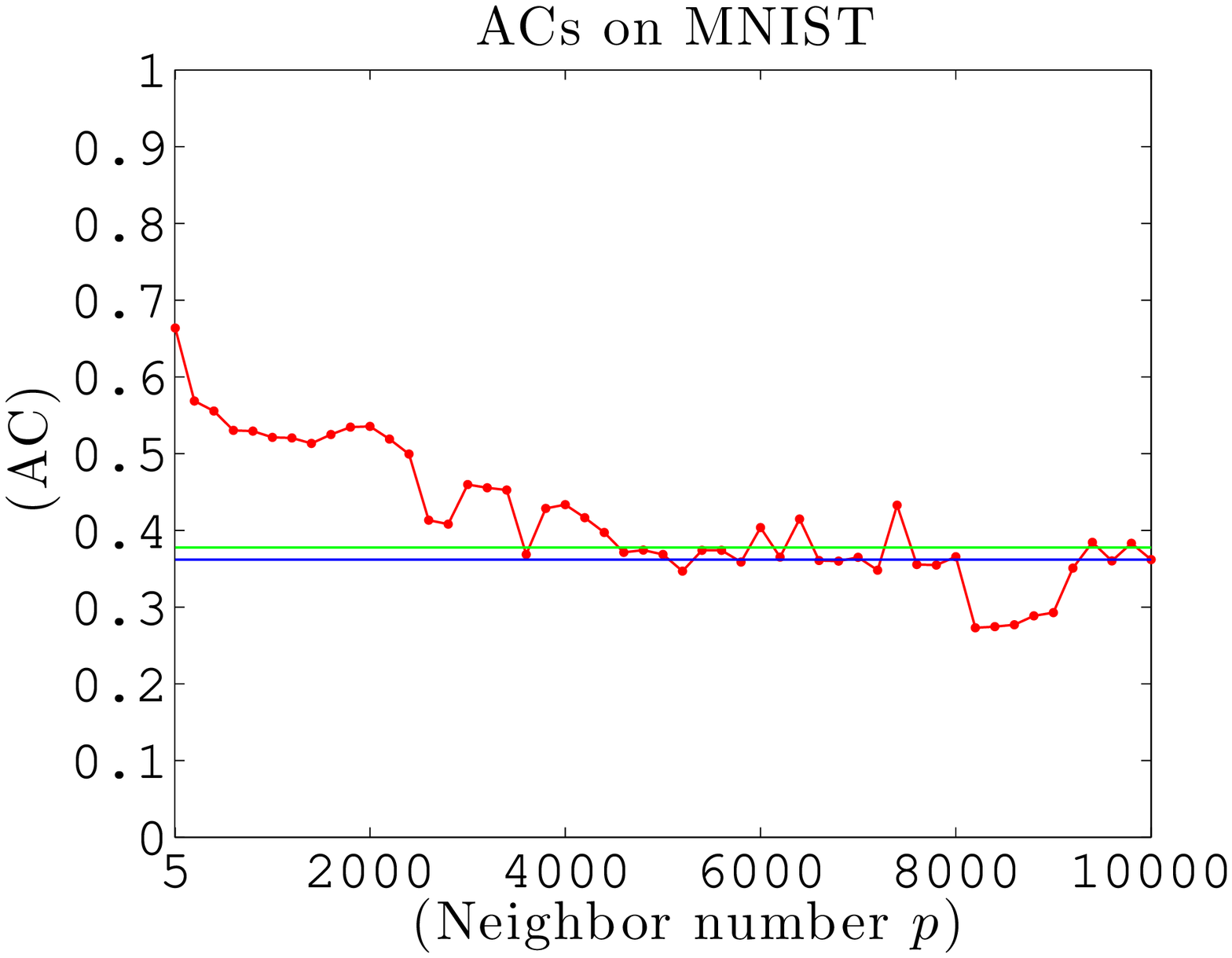}
\end{minipage}
\begin{minipage}{.4\linewidth}
 \includegraphics[width=\linewidth]{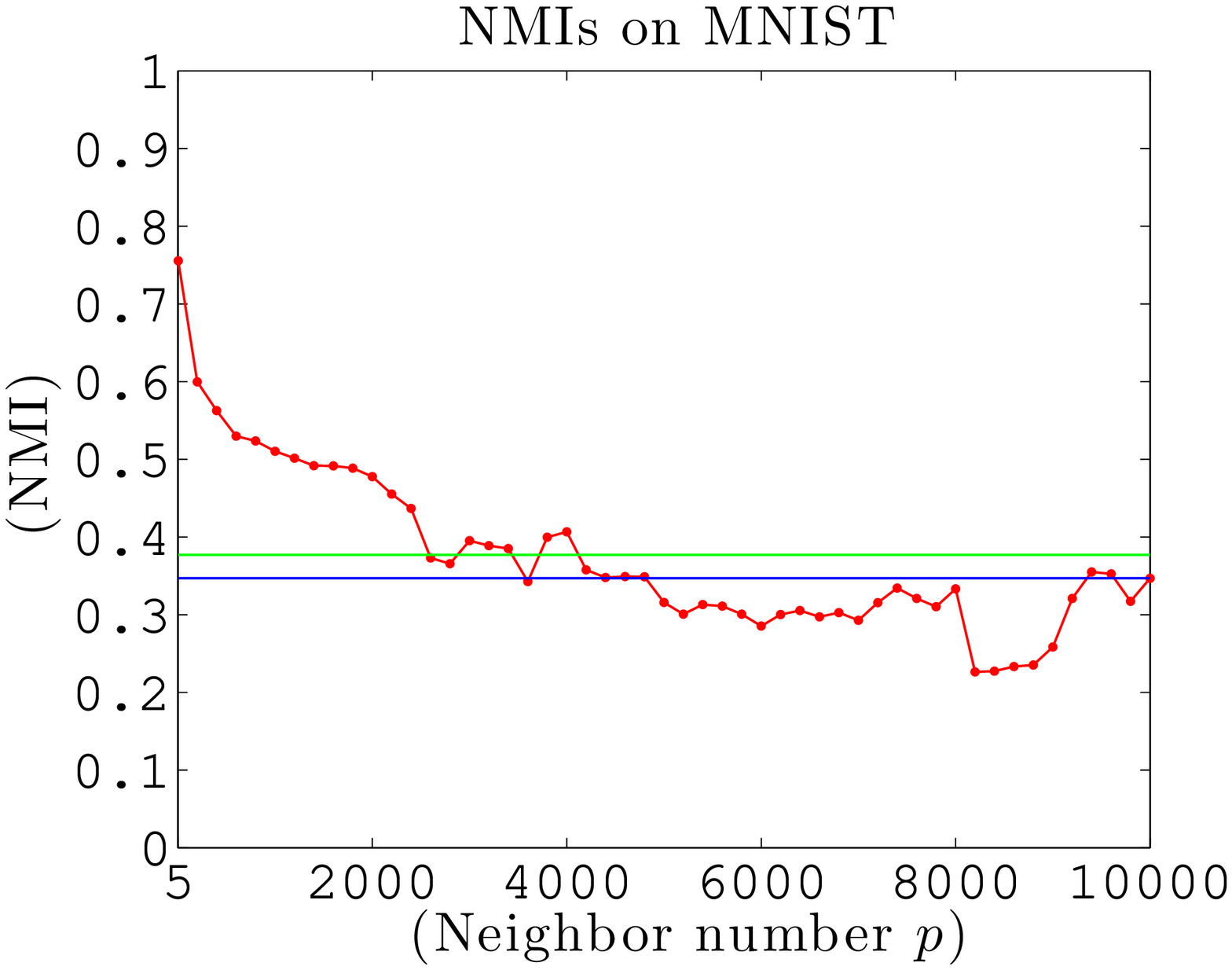}
\end{minipage}
 \smallskip

\begin{minipage}{.4\linewidth}
 \includegraphics[width=\linewidth]{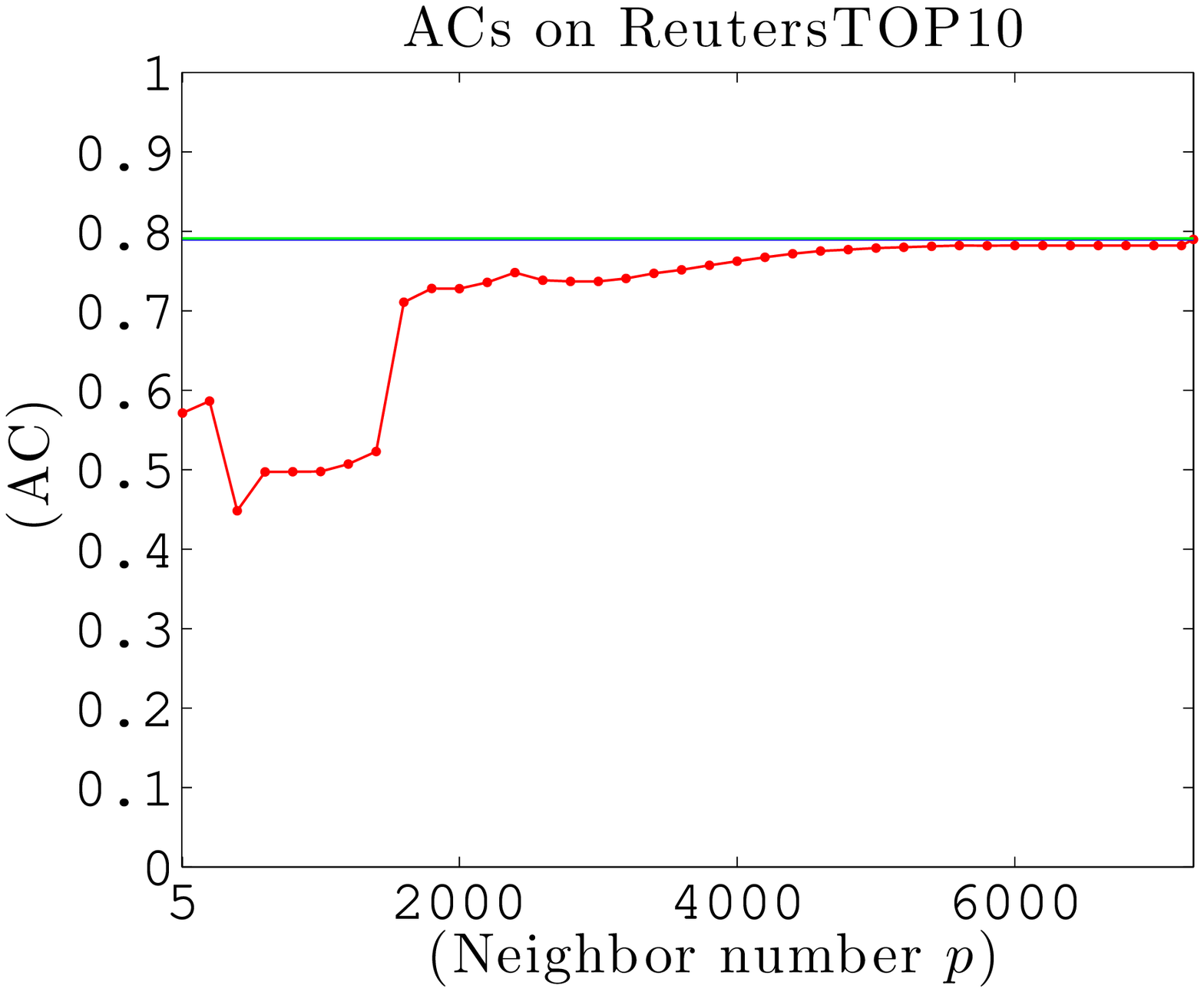}
\end{minipage}
\begin{minipage}{.4\linewidth}
 \includegraphics[width=\linewidth]{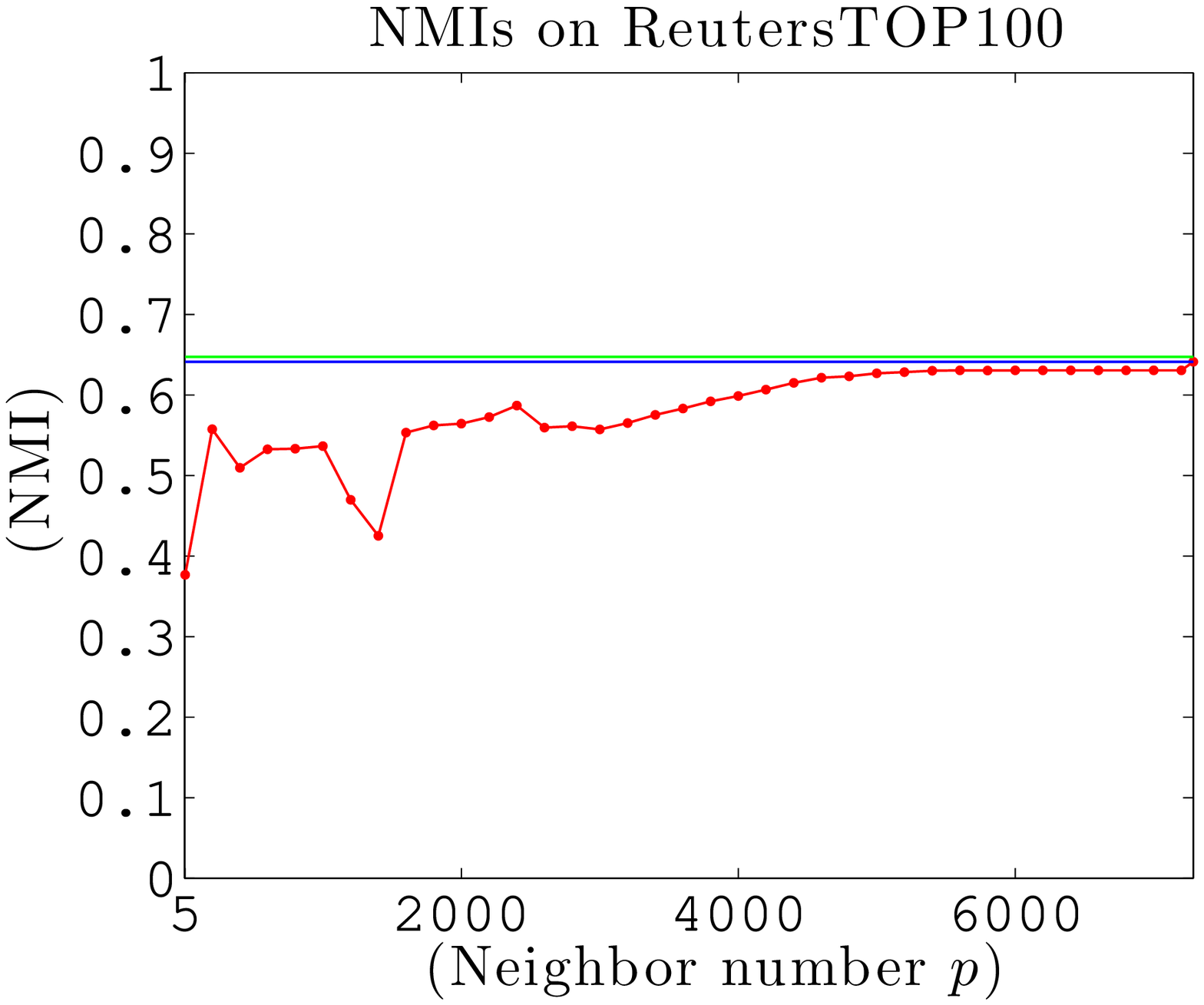}
\end{minipage}
\smallskip

\begin{minipage}{.4\linewidth}
 \includegraphics[width=\linewidth]{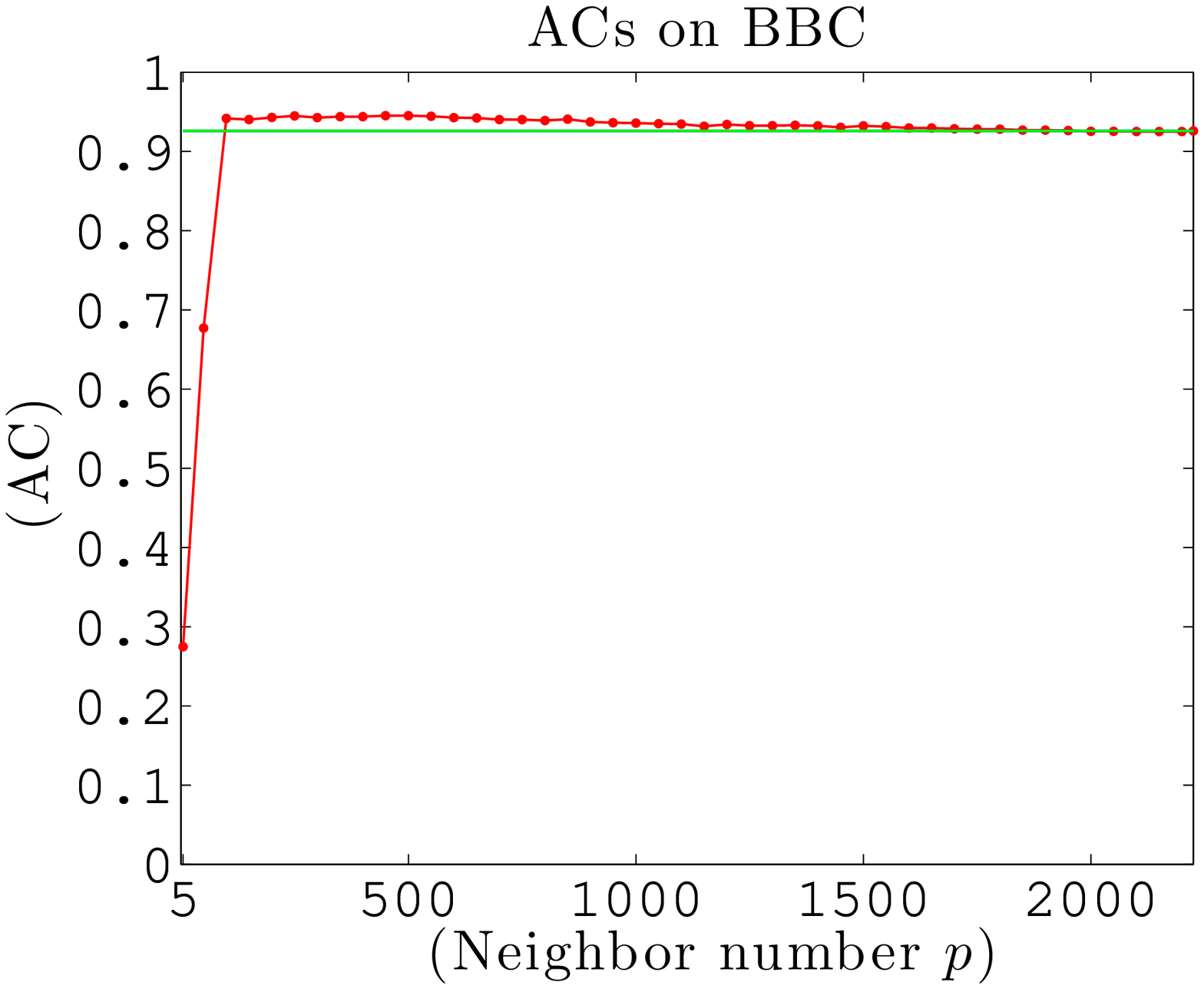}
\end{minipage}
\begin{minipage}{.4\linewidth}
 \includegraphics[width=\linewidth]{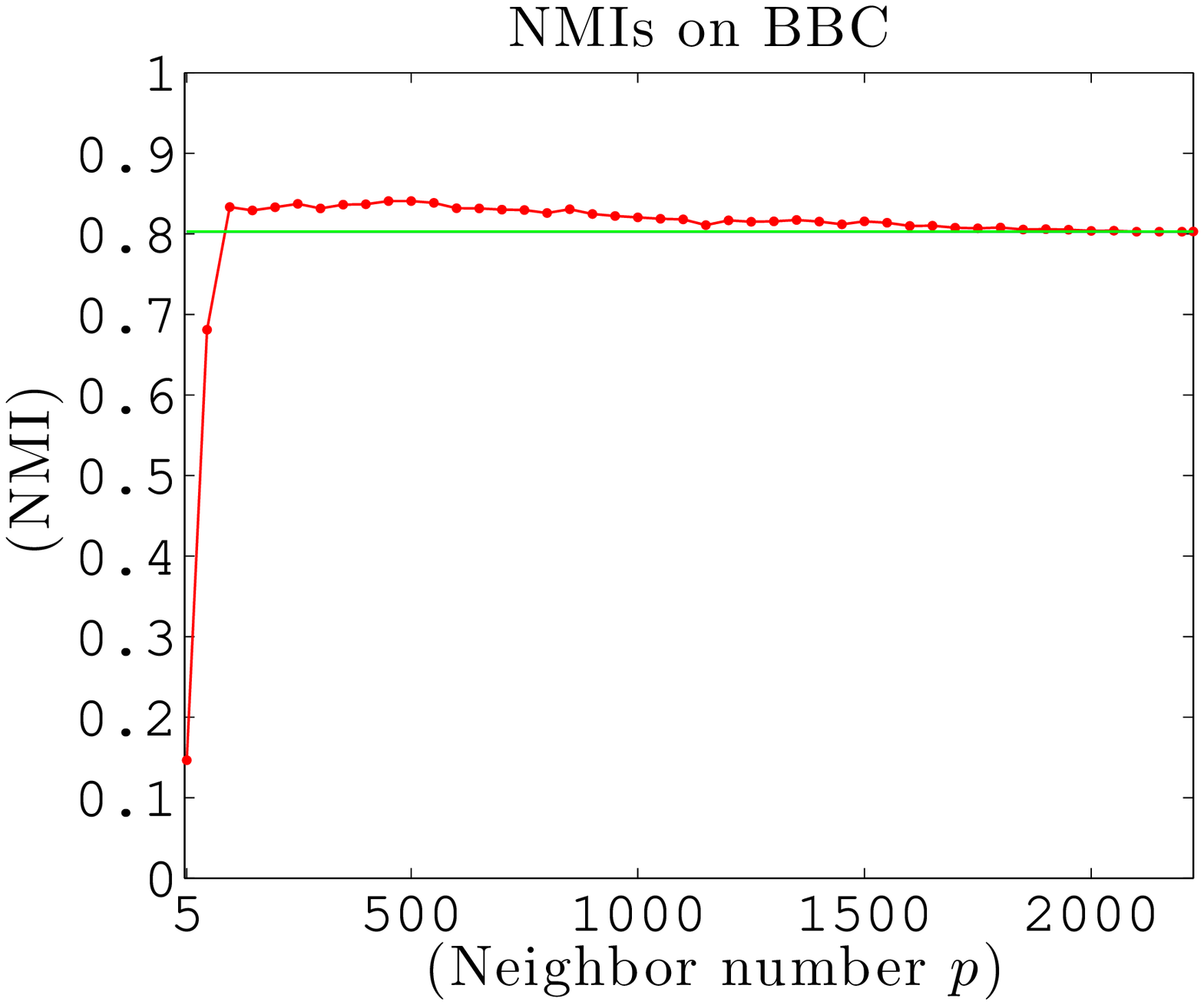}
\end{minipage}

 \caption{AC and NMI of NCER, MER and ER versus neighbor number $p$.
 The graphs from the top to bottom show the ACs and NMIs of the algorithms
 on COIL20, MNIST, ReutersTOP10, and BBC.
 The left graphs are the ACs,  and the right graphs are the NMIs.
 The horizontal axis is the neighbor number $p$, 
 and the vertical axis is the measured value of AC and NMI.
 The red points connected by the red line are for NCER.
 The blue and green lines are for MER and ER, respectively.}
 \label{Fig: Ex3}
\end{figure}

\begin{table}[h]
 \centering
 \caption{ACs and NMIs of NCER with $p=m$, MER, and ER on the data sets.}
 \label{Tab: Exp3}
\begin{tabular}{lccccccccc}
 \hline
 & \multicolumn{2}{c}{NCER with $p=m$} & & \multicolumn{2}{c}{MER} & & \multicolumn{2}{c}{ER} \\
 & AC & NMI & & AC & NMI & & AC & NMI \\
  \cline{2-3} \cline{5-6} \cline{8-9}
 COIL20 & 0.376 & 0.602 & & 0.376 & 0.602 & & 0.444 & 0.639 \\
 MNIST & 0.362 & 0.347 & & 0.362 & 0.347 & & 0.378 & 0.377 \\
 ReutersTOP10 & 0.790 & 0.641 & & 0.790 & 0.641 & & 0.791 & 0.647 \\
 BBC & 0.926 & 0.803 & & 0.926 & 0.803 & & 0.926 & 0.803 \\
 \hline
\end{tabular}
\end{table}

\section{Concluding Remarks} 
\label{Sec: Concluding remarks}

We developed the NCER algorithm for spectral clustering;
it is a variant of the normalized cut algorithm of Shi and Malik and Ng et al.
The similarity with the ER algorithm for a separable NMF was discussed.
In particular, if we modify one step of ER, 
the final outputs of NCER and the modified version of ER, 
called MER, coincide if we place assumptions on the data points and input parameters.
Experiments indicated that NCER is a stable clustering algorithm 
and has high performance.
They also showed how NCER behaves when the neighbor number $p$ was varied.
The results confirmed our theoretical insight that NCER is connected with MER
when $p$ is set to be equal to  the number of data points.

Finally, we should mention the issues which will be addressed in future research.
In the MVEE computation, 
we used a cutting plane framework to accelerate 
the efficiency of the interior-point algorithm.
Thanks to the hybrid of interior-point algorithm and cutting plane algorithm,
we could handle large problems.
However, there is no theoretical guarantee that the hybrid algorithm terminates 
after a finite number of iterations.
The experiments showed that it does not achieve the stopping criteria
under some parameter settings even after many iterations.
Therefore, we might want to consider alternative approaches for the computation.
For instance, the conditional gradient algorithm,
which is also referred to as the Frank-Wolfe algorithm, 
for the dual of MVEE formulation $\MVEEProb(\SC)$ is a promising approach;
\cite{Ahi08} reports encouraging experimental results.
The memory requirements of the algorithm are not so large, and thus, 
it should be able to work on large problems though it needs more iterations
than the interior-point algorithm does.

\bibliographystyle{plain}
\bibliography{main}

\end{document}